\documentclass[lettersize,journal]{IEEEtran}

\usepackage[switch]{lineno}

\usepackage{amsmath,amsfonts,amsthm}

\DeclareMathOperator*{\argmin}{arg\,min}
\usepackage{algorithmic}
\usepackage{algorithm}
\usepackage{array}
\usepackage[caption=false,font=normalsize,labelfont=sf,textfont=sf]{subfig}
\usepackage{textcomp}
\usepackage{stfloats}
\usepackage{url}
\usepackage{verbatim}
\usepackage{graphicx}
\usepackage{cite}
\usepackage{arydshln}
\usepackage{multirow}
\usepackage[normalem]{ulem}

\usepackage{diagbox}

\usepackage{xcolor}
\usepackage{enumitem}
\usepackage{xspace}

\newcommand{\eg}{e.g.,\xspace}
\newcommand{\ie}{i.e.,\xspace}

\theoremstyle{definition}
\newtheorem{definition}{\textbf{Definition}}

\newtheorem{theorem}{\textbf{Theorem}}
\newtheorem{corollary}{\textbf{Corollary}}
\newtheorem*{problem*}{\textbf{Problem Definition}}

\newcommand{\model}{\emph{GNNMoE}\xspace}

\usepackage{pdfpages}    % 关键：拼接外部 PDF
\usepackage{fancyhdr}    % 如需给附录页加统一的页眉页脚

\begin{document}
% \linenumbers
\title{Mixture of Message Passing Experts with Routing Entropy Regularization for Node Classification}

\author{
Xuanze Chen, 
Jiajun Zhou, 
Yadong Li,
Jinsong Chen,
Shanqing Yu, 
Qi Xuan, \IEEEmembership{Senior Member, IEEE}
\thanks{This work was supported in part by National Natural Science Foundation of China under Grants 62503423 and U21B2001, in part by the Key Research and Development Program of Zhejiang under Grants 2022C01018 and 2024C01025, and in part by the Baima Lake Laboratory Joint Fund of Zhejiang Provincial Natural Science Foundation of China under Grant LBMHZ25F020002. \emph{(Corresponding authors: Jiajun Zhou.)}}
% \thanks{Jiajun Zhou is with the Institute of Cyberspace Security, Zhejiang University of Technology, Hangzhou 310023, China, and also with the Binjiang Institute of Artificial Intelligence, ZJUT, Hangzhou 310056, China (e-mail: jjzhou@zjut.edu.cn).}
\thanks{Xuanze Chen, Jiajun Zhou, Yadong Li, Shanqing Yu and Qi Xuan are with the Institute of Cyberspace Security, Zhejiang University of Technology, Hangzhou 310023, China, and also with the Binjiang Cyberspace Security Institute of ZJUT, Hangzhou, 310056, China (e-mail: jjzhou@zjut.edu.cn).}
\thanks{Jinsong Chen is with Faculty of Artificial Intelligence in Education, Central China Normal University, Wuhan 430074, China
(e-mail: guangnianchenai@ccnu.edu.cn)
.}
}

\markboth{Journal of \LaTeX\ Class Files,~Vol.~14, No.~8, August~2021}%
{Shell \MakeLowercase{\textit{et al.}}: A Sample Article Using IEEEtran.cls for IEEE Journals}


\maketitle

% \IEEEdisplaynontitleabstractindextext

% \IEEEpeerreviewmaketitle

\begin{abstract}
  Graph neural networks (GNNs) have achieved significant progress in graph-based learning tasks, yet their performance often deteriorates when facing heterophilous structures where connected nodes differ substantially in features and labels. To address this limitation, we propose \model, a novel entropy-driven mixture of message-passing experts framework that enables node-level adaptive representation learning. \model decomposes message passing into propagation and transformation operations and integrates them through multiple expert networks guided by a hybrid routing mechanism. And a routing entropy regularization dynamically adjusts soft weighting and soft top-$k$ routing, allowing \model to flexibly adapt to diverse neighborhood contexts. Extensive experiments on twelve benchmark datasets demonstrate that \model consistently outperforms SOTA node classification methods, while maintaining scalability and interpretability. This work provides a unified and principled approach for achieving fine-grained, personalized node representation learning.
\end{abstract}
    
\begin{IEEEkeywords}
    Graph Neural Networks, Mixture of Experts, Heterophily, Node Classification, Entropy Regularization
\end{IEEEkeywords}

\section{Introduction}

\IEEEPARstart{I}{n} real-world scenarios, complex interactions such as social contacts, financial transactions, and scientific collaboration can be abstracted as graph-structured data. To mine and exploit latent information from these graphs, researchers have proposed a variety of graph representation learning methods~\cite{cao2016deep,gori2005new}, among which graph neural networks (GNNs)~\cite{corso2024graph,xu2018powerful} have attracted extensive attention due to their strong modeling capacity. Through iterative message passing, GNNs have achieved notable success on node classification tasks such as fraud detection~\cite{Meta-IFD,zhou2022behavior}, social bot detection~\cite{yang2023rosgas} and entity recognition~\cite{zhang2025social}. However, classical GNNs~\cite{GCN,GAT,GraphSAGE} rely heavily on local structure and the homophily assumption, which limits their generalization on heterophilous graphs, sparsely connected graphs, and cross-graph transfer tasks~\cite{sun2024breaking,zhao2025grain}. As a consequence, they struggle to accommodate the diversity of node features, topologies, and label distributions in practice, revealing a lack of broad applicability.

Classical GNNs such as GCN~\cite{GCN}, GraphSAGE~\cite{GraphSAGE}, and GAT~\cite{GAT} typically share parameters across the entire graph and learn node representations via a uniform message-passing scheme.
While this design provides a baseline level of generalization, it overlooks node-specific differences in structural attributes and feature distributions, severely limiting the model's expressive power to capture heterophilous information in complex graphs. To alleviate this, subsequent studies have explored multi-channel spectral filtering~\cite{luan2022revisiting, fagcn2021} and higher-order aggregation~\cite{du2022gbk, H2GCN2020} to strengthen modeling of local heterophily, achieving incremental progress. However, as semantic complexity and topological diversity continue to grow in heterophilous scenarios, existing methods still lack sufficiently fine-grained modeling. On the one hand, multi-channel or higher-order strategies often rely on preset rules or static patterns, making them inflexible to diverse structural characteristics; on the other hand, a single aggregation function still presumes a weight-sharing processing style across nodes, lacking dynamic awareness of individual differences. Consequently, developing more personalized and adaptive representation learning remains a critical direction for enhancing the generalizability and expressiveness of GNNs.

\begin{figure}[!htb]
    \centering
    \includegraphics[width=0.9\linewidth]{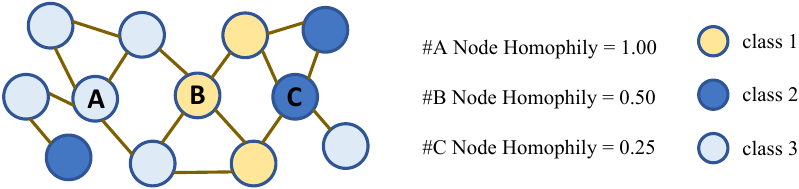}
    \caption{Example of complex neighborhood context in graphs.}
    \label{fig:case1}
\end{figure}

As illustrated in Fig.~\ref{fig:case1}, label distributions within node neighborhoods are non-uniform, reflecting the complexity of behavioral patterns in real-world systems. In social networks~\cite{lee2025sfgcn, fan2019graph}, some users belong to highly homophilous interest communities whose neighbor labels are largely consistent, whereas others exhibit broader interests and thus display pronounced heterophily in neighborhood labels. In co-authorship networks~\cite{liu2022deep, ali2020graph}, researchers focused on a single field predominantly collaborate within that field, forming local structures with high label concentration; conversely, interdisciplinary scholars often collaborate across multiple domains, resulting in highly diverse neighborhood labels. This heterophily in neighborhood label distributions essentially captures the complex semantic contexts surrounding nodes and imposes stronger demands on adaptive modeling capabilities of GNNs.

Regarding this issue, recent studies have proposed an explicit separation learning paradigm~\cite{NCGNN}, which aims to divide the node representation learning process into different encoding subspaces through a metric-driven strategy based on neighborhood information distribution. This paradigm enhances the model's ability to adaptively model the diversity of local information and offers a new pathway toward fine-grained node representations. However, existing separation learning strategies face three critical limitations: 
1) high sensitivity to threshold parameters, which makes performance unstable in high heterophilous scenarios; 
2) coarse-grained subspace partitioning, where each subspace still relies on a uniform message-passing mechanism, preventing fine-grained modeling of individual node differences; 
3) the need for auxiliary computational modules to dynamically compute separation metrics and assign subspaces, substantially increasing training overhead.
These limitations raise a key question: \emph{Can we construct a node-level adaptive message-passing mechanism that allows each node to dynamically select its optimal propagation path and encoding strategy according to its local structure and semantic context, thereby achieving personalized representation and universal modeling on various graphs?}

To address the above problem, we analyze the preferences of nodes with varying degrees of homophily toward different encoding schemes. The specific observations and empirical evidence are presented in Sec.~\ref{sec:Obs}, where we validate the existence of such preferences. Actually, current globally or locally shared message-passing schemes fail to accommodate such fine-grained preferences. These findings provide direct motivation for pursuing node-level adaptive representation learning.
\emph{Therefore, we propose GNNMoE, an entropy-driven mixture of message-passing experts framework tailored for generic node classification.}

\textbf{Contributions:}
\model adopts three key designs: 1) A mixture of message passing expert networks constructed by recombining propagation and transformation operations; 2) A hybrid routing mechanism that incorporates both soft and hard routing to dynamically dispatch expert networks and activation functions; 3) A routing entropy regularization mechanism that dynamically adjusts soft weighting and soft top-$k$ routing, allowing \model to flexibly accommodate diverse encoding preferences. Extensive experiments across homophilous and heterophilous benchmarks demonstrate the effectiveness and superiority of \model. The framework consistently adapts to node-specific structural and semantic variations, achieving fine-grained, personalized representation learning and improved generalization performance.

\section{Preliminaries}
\subsection{Notations} 
A graph is denoted as $G = (V, E, \boldsymbol{X}, \boldsymbol{Y})$, where $V$ and $E$ are the set of nodes and edges respectively, $\boldsymbol{X} \in \mathbb{R}^{|V| \times d}$ is the node feature matrix, and $\boldsymbol{Y} \in\mathbb{R}^{|V| \times C}$ is the node label matrix. 
Here we use $|V|$, $d$ and $C$ to denote the number of nodes, the dimension of the node features, and the number of classes, respectively. 
Each row of $\boldsymbol{X}$ (i.e., $\boldsymbol{x}_i$) represents the feature vector of node $v_i$, and each row of $\boldsymbol{Y}$ (i.e., $\boldsymbol{y}_i$) represents the one-hot label of node $v_i$.
The graph topology information $(V, E)$ can also be denoted by an adjacency matrix $\boldsymbol{A} \in \mathbb{R}^{|V| \times |V|}$, where $\boldsymbol{A}_{ij}=1$ indicates the existence of an edge between $v_i$ and $v_j$, and $\boldsymbol{A}_{ij}=0$ otherwise. Based on the adjacency matrix, we can define the degree distribution of $G$ as a diagonal degree matrix $\boldsymbol{D} \in \mathbb{R}^{|V|\times |V|}$ with entries $\boldsymbol{D}_{ii} = \sum_{j} \boldsymbol{A}_{ij}$ representing the degree value of $v_i$. 
% Node classification is a fundamental task in graph machine learning, and it involves assigning labels to the nodes of a graph based on their features and the graph topology structure.

\begin{figure*}[!htb]
    \centering
    \includegraphics[width=1.0\linewidth]{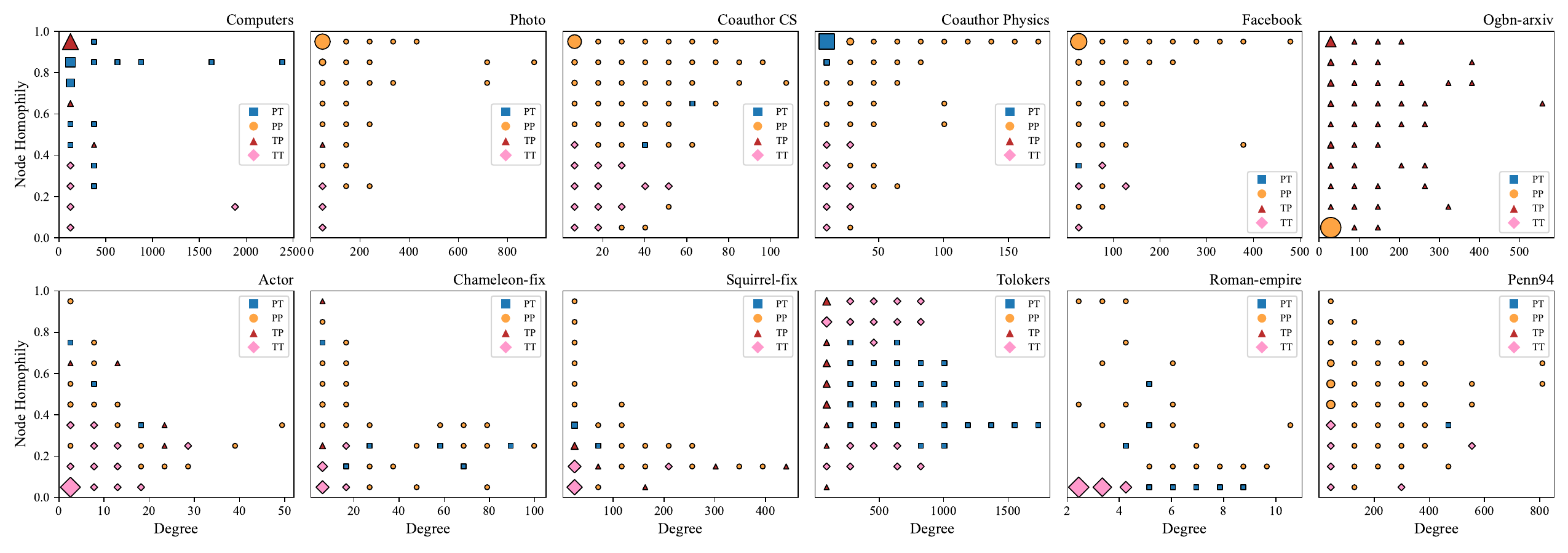}
    \caption{Observation experiment 1. Illustration the preferences of nodes with varying degrees of homophily toward different encoding schemes. Nodes are partitioned into subspaces according to their homophily levels and degrees. Distinct marker shapes highlight the encoding scheme achieving the best node classification performance within each subspace, while the marker size reflects the number of nodes in that subspace.
    }
    \label{fig:ob1}
\end{figure*}

\subsection{Decoupling of Message Passing in GNNs}
From a decoupled perspective, message passing in GNNs can be decomposed into two functionally independent operations, namely Propagation (P) and Transformation (T)~\cite{PT}, which can be formulated as follows:
\begin{equation}
    \begin{aligned}
         &\textbf{Propagation:}\  \boldsymbol{h}_\textit{i}^{(l-1)} \leftarrow \textsf{P}\left(\boldsymbol{h}_\textit{i}^{(l-1)},\  \left\{\boldsymbol{h}_\textit{j}^{(l-1)}\mid j\in\mathcal{N}(i) \right\}\right)\\
         &\textbf{Transformation:}\  \boldsymbol{h}_\textit{i}^{(l)}=\textsf{T}\left(\boldsymbol{z}_\textit{i}^{(l-1)}\right)
    \end{aligned}
\end{equation}
where $\boldsymbol{h}_\textit{i}^{(l)}$ is the node representation during $l$-th message passing, $\mathcal{N}(i)$ is the neighbor set of node $v_\textit{i}$. $\textsf{P}$ is the propagation function that combines message generation and aggregation from neighbor node $v_j$ to target node $v_i$. $\textsf{T}$ performs a non-linear transformation on the state of the nodes after propagation.
Based on the disentanglement, existing GNN architectures can be loosely categorized into four types according to the stacking order of propagation and transformation operations: $\textsf{PTPT}$, $\textsf{PPTT}$, $\textsf{TTPP}$, and $\textsf{TPTP}$.

The propagation operations in message passing admit multiple instantiations, with representative forms including uniform propagation in GCN, attention-weighted propagation in GAT, and aggregator-based propagation in GraphSAGE.

% \textbf{GCN-like Propagation} performs weighted aggregation with symmetry normalization for the features of each node and its neighbors:
\textbf{GCN-like Propagation} performs a symmetrically normalized weighted sum over each node and its neighbors:
\begin{equation}
 \boldsymbol{P}_i = \sum_{v_j\in\mathcal{N}(i)\cup\{v_i\}} \hat{\boldsymbol{A}}_{ij}\boldsymbol{x}_j
\end{equation}
where $\hat{\boldsymbol{A}} = (\boldsymbol{D}+\boldsymbol{I})^{-\frac12}\,(\boldsymbol{A}+\boldsymbol{I})\,(\boldsymbol{D}+\boldsymbol{I})^{-\frac12}$. 
This propagation is essentially a diffusion-style smoothing over the graph, balancing the influence of high- and low-degree nodes and driving neighboring node features toward convergence.

\textbf{GraphSAGE-like Propagation} performs a statistical or learnable summary of a node's neighborhood:
\begin{equation}
    \boldsymbol{P}_i=\operatorname{AGG}\left(\left\{\boldsymbol{x}_j: v_j \in \mathcal{N}(i)\right\}\right)
\end{equation}
where $\operatorname{AGG}$ can be a mean, max-pooling, or an LSTM-based aggregator. This propagation emphasizes extracting a representative neighborhood statistic, enabling the model to capture local context while preserving node-level individuality.

\textbf{GAT-like Propagation} performs adaptive, neighbor-specific weighting via an attention mechanism:
\begin{equation}
    \boldsymbol{P}_i=\sum_{v_j \in \mathcal{N}(i) \cup\left\{v_i\right\}} \alpha_{ij} \boldsymbol{x}_j
\end{equation}
where $\alpha_{i j}$ are attention coefficients learned from feature pairs. This propagation adaptively emphasizes task-relevant neighbors and suppresses noisy or irrelevant connections.

\textbf{Transformation} refers to applying a learnable transformation to each node feature, independent of the underlying graph structure. Representative instantiations include linear projection followed by a nonlinear activation $\sigma$ and dropout regularization as follows:
\begin{equation}
    \boldsymbol{T}_i=\mathrm{Dropout}\left(\sigma(\boldsymbol{W}(\boldsymbol{x}_i))\right)
\end{equation}
where $\boldsymbol{W}(\cdot)$ is the linear projection weight matrix, $\sigma$ is the nonlinear activation function.

\subsection{Mixture of Experts Architecture}
The Mixture of Experts (MoE) architecture scales neural network capacity by combining many specialized sub-networks, called \emph{experts}, while only activating part of them for each input. Given an input $\boldsymbol{x}$, a lightweight routing network produces routing scores:
\begin{equation}
    \boldsymbol{\pi}=\operatorname{Route}(\boldsymbol{x}) \in \mathbb{R}^M
\end{equation}
where $M$ is the number of experts. Each expert $E_i(\cdot)$ processes $\boldsymbol{x}$ and produces an output $\boldsymbol{y}_i$. The final output is a weighted combination of the outputs from the selected experts:
\begin{equation}
    \boldsymbol{y}=\sum_{i=1}^M \pi_i(\boldsymbol{x})\cdot E_i(\boldsymbol{x})
\end{equation}
where the routing weights $\pi_i(\boldsymbol{x})$ are derived from $\boldsymbol{\pi}$ through a normalization (\eg softmax).
% 专家选择机制通常有两种，软路由和硬路由。对于前者，所有专家贡献，但是权重不同，公式化为；对于后者，为了减少计算，仅仅topk的专家被激活选择，形式化为
The expert selection mechanism typically falls into two categories: soft routing and hard routing. For the former, all experts contribute but with different weights $\pi_i(\boldsymbol{x})$, which can be formalized as follows:
\begin{equation}
    \pi_i(\boldsymbol{x})=\frac{\exp \left(\boldsymbol{\pi}_i\right)}{\sum_{j=1}^M \exp \left(\boldsymbol{\pi}_j\right)}
\end{equation}
For the latter, to reduce computation, only the top-$k$ experts are activated and selected, which can be formalized as follows:
\begin{equation}
    \pi_i(\boldsymbol{x})= \begin{cases}\frac{\exp \left(\boldsymbol{g}_i\right)}{\sum_{j \in \operatorname{top-k}(\boldsymbol{g})} \exp \left(\boldsymbol{g}_j\right)}, & i \in \operatorname{top-k}(\boldsymbol{g}) \\ 0, & \text { otherwise. }\end{cases}
\end{equation}

%%%%%%%%%%%%%% Observation %%%%%%%%%
\section{Observation}\label{sec:Obs}
To explore the preferences of nodes with varying degrees of homophily toward different encoding schemes, we first decompose the message-encoding mechanism of GNNs into two fundamental operations: Propagation (\textsf{P}) and Transformation (\textsf{T}). Since mainstream GNNs can be viewed as compositions of these two operations, we obtain four distinct composite encoding schemes: \textsf{PT}, \textsf{TP}, \textsf{TT}, \textsf{PP}. Specifically, \textsf{PT} denotes propagation followed by transformation, \textsf{TP} denotes transformation followed by propagation, \textsf{PP} denotes two consecutive propagations, and \textsf{TT} denotes two consecutive transformations. We conduct observational experiments on 12 datasets (six homophilous graphs and six heterophilous graphs) to analyze node-level encoding preferences, as illustrated in Fig.~\ref{fig:ob1}.
In these experiments, nodes are partitioned into subspaces according to their homophily levels and degrees. Distinct marker shapes highlight the encoding scheme achieving the best node classification performance within each subspace, while the marker size reflects the number of nodes in that subspace. A clear trend emerges: regardless of whether the graph is homophilous or heterophilous, high-homophily nodes tend to prefer encoding schemes involving propagation operations (\ie \textsf{PT}, \textsf{TP}, \textsf{PP}), whereas low-homophily nodes favor schemes that rely solely on transformation (\textsf{TT}).
A reasonable explanation for this phenomenon is that nodes with high homophily exhibit smoother neighborhood feature distributions, where propagation facilitates the aggregation of homophilous information. In contrast, low-homophily nodes encounter diverse intra- and cross-community feature distributions, for which nonlinear transformation helps extract salient information and filter noise. Furthermore, even nodes with similar homophily levels may favor different encoding schemes, and these differences become more evident in low-homophily scenarios, especially in heterophilous graphs such as Chameleon-fix, Squirrel-fix, and Tolokers.
Finally, different node subspaces across datasets (even among homophilous graphs) exhibit markedly diverse encoding preferences, as observed in Computers, Photo, and Ogbn-arxiv.
In summary, these empirical observations indicate that nodes situated in different neighborhood context display distinct preferences toward encoding schemes. Achieving universal and high-performance node classification thus requires message-passing and encoding mechanisms that are adaptively conditioned on the neighborhood context.

\begin{figure*}[htbp]
    \centering 
    \includegraphics[width=\linewidth]{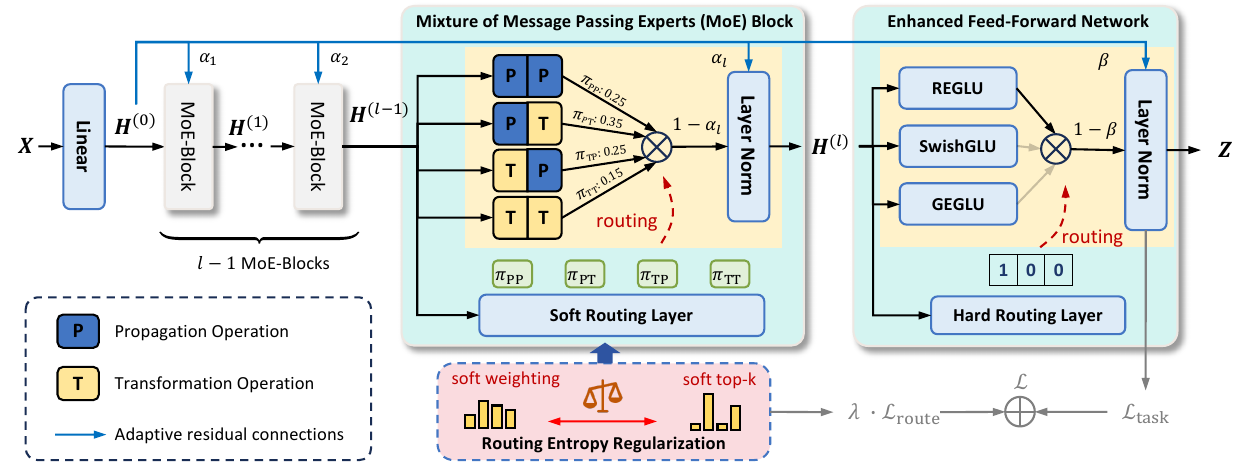}
    \caption{Illustration of \model architectures. The complete workflow proceeds as follows:
    % 1) the MoE network processes information from target nodes using different message-passing schemes;
    % 2) the soft gating layer estimates the contribution of each message-passing expert based on node characteristics, selects the relevant experts, and aggregates their outputs;
    % 3) the entropy-constrained gating sharpening mechanism guides the soft routing strategy to more precisely allocate message passing experts to different nodes in different types of graphs, enhancing the generalization ability for node representation learning; 
    % 3) the hard gating layer selects the appropriate activation layers to enhance the expressive power of the FFN.
    1) each node is processed by multiple message-passing experts, producing diverse candidate representations based on different encoding strategies;
    2) a soft routing network computes routing scores conditioned on the node's features, aggregates the experts accordingly, and produces a preliminary multi-expert representation;
    3) an entropy-driven routing adapter dynamically adjusts the routing process, striking a balance between fully weighted aggregation and approximate top-k expert activation;
    4) the aggregated representation is refined by an enhanced FFN with hard routing that adaptively selects the activation function, jointly improving the expressiveness of final node representation.
    }
    \label{fig:framework}
\end{figure*}

%%%%%%%%%%%%%%%%%%%%%%%%%%%%%%%%%%%%%%%%%%%%%%%%%%%%%%%%%%%%%%%%%%%%%%
\section{Methodology}\label{sec:method}
Based on the above observations, we note that nodes within a single graph or across multiple graphs exhibit varying levels of homophily or heterophily, and show preferences toward specific encoding strategies. To achieve node-level adaptive representation learning, we introduce a mixture-of-experts (MoE) architecture into graph neural networks. By designing expert networks with diverse encoding paradigms and equipping them with a flexible expert routing mechanism, we develop a \model framework. 
The core components of \model are as follows:
1) A \textbf{mixture of message-passing experts block} integrates multiple experts with diverse encoding paradigms, while a \textbf{soft routing layer} computes node-specific routing scores to adaptively fuse expert contributions;
2) An \textbf{entropy-driven routing adapter} dynamically adjusts soft weighting or soft top-$k$ routing to flexibly coordinate different encoding strategies;
4) An \textbf{enhanced feed-forward network (EFFN)} with \textbf{hard routing layer} refines aggregated outputs and adaptively selects the most suitable activation function to boost expressiveness.
The complete framework is illustrated in Fig.~\ref{fig:framework}.

\subsection{General GNNMoE Architecture}
From a macro perspective, \model is composed of stackable MoE-blocks together with an enhanced FFN (EFFN) and an entropy-driven routing adapter. It takes node features and adjacency information as input and produces the final node representations as output. From a micro perspective, each MoE-block consists of a message-passing expert network and a soft routing layer, while the EFFN is composed of a standard FFN augmented with a hard routing layer. The design details of each module will be introduced below.

\subsubsection{Mixture of Message-Passing Experts Block}
This module incorporates a mixture of message-passing experts, each constructed from distinct combinations of propagation and transformation operations. These experts represent diverse encoding paradigms, allowing \model to capture various structural patterns across different graphs. By assembling multiple experts into a unified module, the module provides a flexible basis for node-level adaptive representation learning.

First, the input features $\boldsymbol{X}$ will be transformed into an initial feature embedding through a linear transformation parameterized by $\boldsymbol{W}_0\in\mathbb{R}^{d\times d^\prime}$ and a ReLU activation:
\begin{equation}
  \boldsymbol{H}^{(0)}=\operatorname{ReLU}\left(\boldsymbol{X}\boldsymbol{W}_0\right)
\end{equation}
where $d^\prime$ is the hidden dimension.
Next, we stack several mixture of message-passing experts blocks, called MoE-blocks, to further learn node representations. Each MoE-block consists of a message passing expert network $\mathcal{E}(\cdot)$ and a soft routing layer $\operatorname{SR}(\cdot)$, where $\mathcal{E}=\left\{\textsf{PP}, \textsf{PT}, \textsf{TP}, \textsf{TT}\right\}$ contains four message passing experts specialized in handling graph features from different neighborhood contexts. 

For the $(l)$-th MoE-block, it takes the node representation output from the $(l-1)$-th MoE-block as input, then computing the routing scores through the soft routing layer: 
\begin{equation}\label{eq: pi}
\begin{aligned}
    \boldsymbol{\pi}&=\operatorname{SR}\left(\boldsymbol{H}^{(l-1)}\right) \\
    &= \operatorname{Softmax}\left(\boldsymbol{W}_2^{(l)} \cdot \operatorname{ReLU}\left( \boldsymbol{H}^{(l-1)}\boldsymbol{W}^{(l)}_1\right)\right)
\end{aligned}
\end{equation}
where $\boldsymbol{\pi}=\left\{\pi_{1},\pi_{2},\pi_{3},\pi_{4} \right\}\in\mathbb{R}^4$ is the routing weight vector, $\boldsymbol{W}_1^{(l)}$ and $\boldsymbol{W}_2^{(l)}$ are the transformation weights.
Next, the graph messages processed by different experts are aggregated via routing weights, and an initial multi-expert representation is generated through adaptive residual connections:
\begin{equation}
 \begin{aligned}
 &\boldsymbol{H}^{(l-1)} \leftarrow \sum_{i=1}^{4} \pi_{i} \cdot \mathcal{E}_{i}\left(\boldsymbol{A}, \boldsymbol{H}^{(l-1)}\right) \\
 &\boldsymbol{H}^{(l)}= \operatorname{LN} \left( \alpha_l \cdot \boldsymbol{H}^{(0)} + (1 - \alpha_l) \cdot \boldsymbol{H}^{(l-1)} \right)
\end{aligned}
\end{equation}
where $\operatorname{LN}(\cdot)$ denotes the layer normalization operation, and $\alpha_\textit{l}$ is a learnable parameter that controls the adaptive initial residual connection.

\subsubsection{Enhanced Feed-Forward Network}
After message passing via $l$ MoE-blocks, \model effectively fuses node attribute information with topological structure. To further improve the expressiveness of node representations, and inspired by the role of FFNs in Graph Transformer architectures, we design an EFFN module within \model. Specifically, the EFFN consists of a hard routing layer $\operatorname{HR}(\cdot)$ and a mixture of activation experts $\mathcal{A}=\left\{\text{SwishGLU},\text{GEGLU},\text{REGLU} \right\}$. 
Each activation expert offers distinct benefits:
SwishGLU~\cite{SwishGLU} combines Swish activation with gating to facilitate stable gradient propagation; 
GEGLU~\cite{SwishGLU} introduces gated additive activations to enrich nonlinear expressiveness; 
and REGLU~\cite{SwishGLU} extends ReLU with gating to alleviate gradient vanishing while maintaining computational efficiency. 

Specifically, the multi-expert representations of node features encoded by $l$ MoE-blocks are then input into a hard routing layer, which adaptively selects the most suitable activation function for further feature encoding:
\begin{equation}
 j=\operatorname{HR}\left(\boldsymbol{H}^{(l)}\right)=\operatorname{Gumbel\_Softmax}\left(\boldsymbol{H}^{(l)}\right)
\end{equation}
where $j \in\{1,2,3\}$. The selected expert will encode $\boldsymbol{H}^{(l)}$ to enhance its expressiveness, followed by an adaptive residual connection to generate the final node representation:
\begin{equation}
 \begin{aligned}
  &\boldsymbol{Z} \leftarrow \mathcal{A}_\textit{j}\left(\boldsymbol{H}^{(l)}\right) = \left( \sigma_\textit{j}\left(\boldsymbol{H}^{(l)}\boldsymbol{W}_3\right)\otimes \boldsymbol{H}^{(l)}\boldsymbol{W}_4 \right)\boldsymbol{W}_5\\
 &\boldsymbol{Z} = \operatorname{LN}\left(\beta \cdot \boldsymbol{H}^{(0)}+(1-\beta)\cdot \boldsymbol{Z}\right)
 \end{aligned}
\end{equation}
where $\sigma\in \left\{\text{Swish},\text{GELU},\text{ReLU}\right\}$, $\boldsymbol{W}_3$, $\boldsymbol{W}_4$, $\boldsymbol{W}_5\in\mathbb{R}^{d^\prime\times d^\prime}$ are the transformation weights, $\otimes$ is the element-wise multiplication, $\beta$ is a learnable parameter that controls the adaptive residual connection.

\begin{figure}[!htp]
    \centering
    \includegraphics[width =\linewidth]{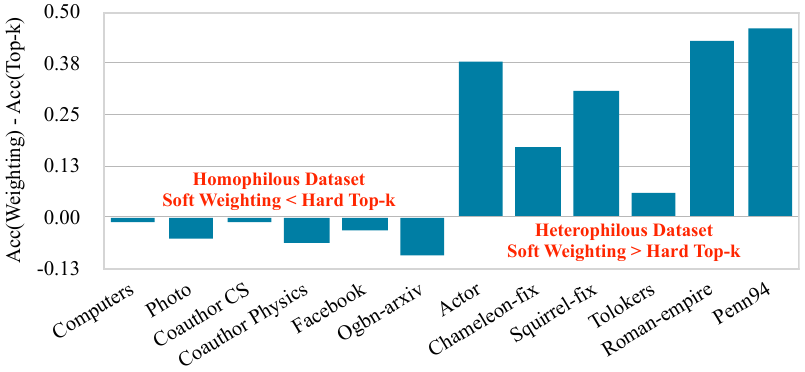}
    \caption{Observation experiment 2. Preference for expert routing strategies across different graphs.}
    \label{fig:obs2}
\end{figure}

\subsubsection{Entropy-driven Routing Adapter}
Existing MoE studies typically adopt either hard top-$k$ routing~\cite{GMoE,BP-MoE} or soft weighting~\cite{Link-MoE} for expert selection. However, our observation experiments\footnote{The hyperparameter settings for Observation 2 follow Section~\ref{sec: experiment}. For hard top-$k$ routing, $k\in\{1,2,3\}$, and the reported results correspond to the optimal hyperparameter configuration.} reveal that these two routing strategies suit different scenarios and thus have limited generalizability. Specifically, by replacing the soft weighting routing in \model with a hard top-$k$ routing and comparing node classification performance across different graphs, as shown in Fig.~\ref{fig:obs2}, we observe that routing preferences vary by graph type. Homophilous graphs tend to favor hard top-$k$ routing, as their nodes share similar features and structures, so a few dominant encoding paradigms are often sufficient, while combining too many may introduce noise. In contrast, heterophilous graphs prefer soft weighting, since their nodes exhibit diverse attributes and structural patterns, making it necessary to integrate multiple encoding paradigms to achieve expressive representations.

Motivated by the above observation, we argue that the default soft weighting routing strategy in \model remains limited.
To achieve adaptive routing strategies, we aim to make the routing weight distribution produced by the softmax operation in Eq.~(\ref{eq: pi}) adjustable. Specifically, in heterophilous scenarios, the model should generate a relatively smooth routing weight distribution, enabling multiple experts to participate jointly in inference and thus fully exploiting their complementary capabilities. Conversely, in homophilous scenarios, only a few dominant experts are needed for inference, where the routing weight distribution should be sharper to suppress noise.
A straightforward idea is to control the smoothness of the routing weights by adjusting the temperature parameter of the softmax function. However, this practice suffers from an inherent limitation: when the temperature changes from $\tau$ to $\tau'$, the routing layer can simply rescale all routing logits by a factor of $\tau'/\tau$, completely offsetting the effect of the temperature change on the forward propagation. This implies that merely tuning the temperature cannot effectively control the routing weight distribution.

To address this issue, we observe a strong connection between the temperature and entropy. The routing weights $\boldsymbol{\pi}$ output by the soft routing layer can be regarded as a probability distribution over expert selection, whose entropy reflects the dispersion of the routing weights: a higher entropy indicates joint reliance on multiple experts, whereas a lower entropy suggests a stronger preference for a few dominant experts. 
Based on this, we propose a routing entropy regularization mechanism. Specifically, we compute the mean entropy of routing weights across all nodes and layers and incorporate it into the training objective as a regularization term:
\begin{equation}
    \mathcal{L}_\text{route}=-\frac{1}{l\cdot |V|}\sum_{i=1}^{|V|}\sum_{t=1}^{l}\sum_{g=1}^{4}\pi_{g}^{i,t}\cdot \log \pi_{g}^{i,t}
\end{equation}
where $g$ indexes message-passing experts.
By minimizing $\mathcal{L}_\text{route}$, we can guide the soft routing layer to assign higher weights for a subset of experts, thereby sharpening the routing weight distribution to approximate a soft top-$k$ routing. 
We introduce a hyperparameter $\lambda\in\{0.001, 0.01, 0.1, 1\}$ to control the strength of the routing entropy regularization, enabling flexible adjustment between soft weighting and soft top-$k$ routing. This design allows dynamic transitions between expert specialization and collaborative behavior.

\subsubsection{Model Training}
For node classification, we append a prediction head $f_\text{pred}$ parameterized by $\boldsymbol{W}_6\in\mathbb{R}^{d^\prime \times C}$ followed by a Softmax activation to obtain the node predictions. During model training, cross-entropy classification loss is used as the main optimization objective.
\begin{equation}
\begin{aligned}
\hat{\boldsymbol{Y}}&=\operatorname{Softmax}\left( \boldsymbol{Z}\boldsymbol{W}_6 \right)  \\
 \mathcal{L_\text{task}} &= -\operatorname{trace}\left(\boldsymbol{Y}_\text{train}^\top \cdot \log \hat{\boldsymbol{Y}}_\text{train} \right) 
\end{aligned}
\end{equation}
where the trace operation $\operatorname{trace}\left( \cdot \right)$ is used to compute the sum of the diagonal elements of the matrix. Furthermore, we combine the task loss $\mathcal{L}_\text{task}$ with the routing entropy regularization loss $\mathcal{L}_\text{route}$ to form the final optimization objective:
\begin{equation}
 \mathcal{L} = \mathcal{L}_\text{task} + \lambda \cdot \mathcal{L}_\text{route}
\end{equation}

Unlike prior MoE methods, in which hard top-$k$ routing requires manual configuration of $k$ and lacks flexibility, while soft weighting may dilute specialized experts, our entropy-driven soft routing mechanism unifies the strengths of both. By adjusting the regularization coefficient $\lambda$ and minimizing $\mathcal{L}$, the routing strategy transitions between soft weighting and soft top-$k$. This unified design improves performance across diverse graphs and enhances generalizability.

\subsection{Routing Entropy Regularization Theory}\label{sec:obsAndInt}
To further substantiate the design motivation of our entropy-driven routing adapter, we provide a theoretical analysis in this subsection. Specifically, in the routing mechanism of \model, for each MoE-block (assume it is the $t$-th), the soft routing layer produces a set of routing weights $\boldsymbol{\pi}^{i,t}=\{\pi_g^{i,t}\}_{g=1}^m$ over the expert set $\mathcal{E}=\{E_1,\ldots,E_m\}$ for any node $v_i$. 
Before introducing entropy regularization, the routing mechanism of \model corresponds to \emph{soft weighting} (\ie weighted aggregation of multiple experts). After incorporating entropy regularization, we expect \model to adaptively transition between ``soft weighting'' and ``soft top-$k$'' (a few experts dominate) across different nodes and graphs. In particular, as the strength of entropy regularization increases, the routing strategy should effectively approximate soft top-$k$ routing. To theoretically establish that entropy regularization of routing weights enables adaptive control between soft weighting and soft top-$k$, we first decompose the complex model into a node-wise, layer-wise surrogate optimization problem.
Due to space limitations, we provide only a concise version of the theoretical proof in the main text, please refer to Appendix~\ref{app: proof} for more details.

\begin{problem*}
    For a given node, its aggregated output in a MoE-block is $\boldsymbol{H}' =  \sum_i \pi_i \cdot \mathcal{E}_i\left(\boldsymbol{A}, \boldsymbol{H}\right)$, where $\boldsymbol{\pi}=[ \pi_1, \cdots, \pi_m ]\in\Delta^{m}$ are the routing weights from the soft routing layer, $\Delta^{m}:=  \{  \boldsymbol{\pi} \in\mathbb{R}^m : \pi_g\geq0, \sum_g \pi_g=1\}$. The overall training objective is
    \begin{equation}\label{eq: ori-optim-object}
    \min _{\boldsymbol{\pi} \in \Delta^m} \mathcal{L}(\pi)=\underbrace{\mathcal{L}_{\text {task }}\left(\boldsymbol{H}^{\prime}(\boldsymbol{\pi})\right)}_{\text {task loss}}+\lambda \underbrace{H(\boldsymbol{\pi})}_{\text {routing entropy}} 
    \end{equation}
    Since solving the above optimization directly is difficult, we adopt a local surrogate strategy. Specifically, during a routing parameter update, we freeze all parameters except $\boldsymbol{\pi}$ and take a first-order approximation of $\mathcal{L}_\text{task}$ at iteration $t$:
    \begin{equation}\label{eq: 1-order-approx}
        \mathcal{L}_{\text {task }}(\boldsymbol{\pi}) \approx \mathcal{L}_{\text {task }}\left(\boldsymbol{\pi}^t\right)+\left\langle\boldsymbol{\ell}, \boldsymbol{\pi}-\boldsymbol{\pi}^t\right\rangle+o\left(\left\|\boldsymbol{\pi}-\boldsymbol{\pi}^t\right\|\right)
    \end{equation}
    where $\boldsymbol{\ell}=\nabla_{\boldsymbol{\pi}} \mathcal{L}_{\text {task }}\left(\boldsymbol{\pi}^t\right)$. This gradient can be interpreted as encoding which experts are more favorable for reducing the task loss. For clarity, define $u_i:=-\ell_i$ as the instantaneous gain from selecting expert $i$. 
    Substituting Eq.~(\ref{eq: 1-order-approx}) into Eq.~(\ref{eq: ori-optim-object}) and dropping constants independent of $\boldsymbol{\pi}$, the surrogate optimization problem at step $t$ becomes:
    \begin{equation}
    \min _{\boldsymbol{\pi} \in \Delta^m}\langle \boldsymbol{\ell}, \boldsymbol{\pi}\rangle+\lambda H(\boldsymbol{\pi}) \quad \Leftrightarrow \quad \min_{\boldsymbol{\pi} \in \Delta^m} \langle -\boldsymbol{u}, \boldsymbol{\pi}\rangle+\lambda H(\boldsymbol{\pi}) 
    \end{equation}
\end{problem*}

\begin{theorem}[\textbf{Temperature property of entropy-driven routing}]
    \emph{Suppose there are $m\ge 2$ message-passing experts and the routing weight distribution over experts is $\boldsymbol{\pi}=[\pi_1,\ldots,\pi_m]\in\Delta^{m}$, where feasible region
    $\Delta^{m}:=\{\boldsymbol{\pi}\in\mathbb{R}^m, \pi_g\ge 0, \sum_{g=1}^m \pi_g=1\}$.
    For a node $v_i$ at a given MoE-block, each expert has an instantaneous gain $u_g\in\mathbb{R}$ and $\boldsymbol{u}=[u_1, \ldots, u_m] \in \mathbb{R}^m$. Consider the following routing optimization problem:} 
    \begin{equation}
      \min _{\boldsymbol{\pi} \in \Delta^m}\langle -\boldsymbol{u}, \boldsymbol{\pi}\rangle+\lambda H(\boldsymbol{\pi}) ,
    \end{equation}
    \emph{where $\lambda>0$ is the entropy regularization coefficient. Then the optimal routing $\pi_g^{t+1}(\lambda)$ has a unique closed-form solution.}
  \end{theorem}
  \begin{proof}
    To ensure consistency between the surrogate objective and the original objective, while maintaining descent during optimization, we further introduce a trust-region constraint based on KL divergence. At the $t$-th iteration, with step size $\eta>0$, the optimal routing is obtained by solving the following optimization problem:
    \begin{equation}
      \boldsymbol{\pi}^{t+1}=\argmin_{\boldsymbol{\pi} \in \Delta^m}\left\langle -\boldsymbol{u}^t, \boldsymbol{\pi}\right\rangle+\lambda H(\boldsymbol{\pi})+\frac{1}{\eta} \mathrm{KL}\left(\boldsymbol{\pi} \| \boldsymbol{\pi}^t\right).
    \end{equation}
    Write the optimization objective as:
    \begin{equation}
        J(\boldsymbol{\pi})=-\left\langle \boldsymbol{u}^t, \boldsymbol{\pi}\right\rangle-\lambda \sum_g {\pi}_g \log {\pi}_g+\frac{1}{\eta} \sum_g {\pi}_g \log \frac{{\pi}_g}{\pi_g^t},
    \end{equation}
    Since the Hessian matrix of $J(\boldsymbol{\pi})$ is positive definite in the interior of the feasible region, there exists a unique minimizer. We solve for the stationary point using the following Lagrangian function:
    \begin{equation}
    \mathcal{L}(\boldsymbol{\pi}, {\nu})=J(\boldsymbol{\pi})+{\nu}\left(\sum_g \pi_g-1\right)-\sum_g \mu_g \pi_g,
    \end{equation}
    where ${\nu}$ and $\mu_g$ are the multipliers of the equality and inequality constraints. Taking the partial derivative with respect to $\pi_g$ and setting it to zero yields the stationarity condition:
    \begin{equation}
      -u_g^t-\lambda\left(1+\log \pi_g\right)+\frac{1}{\eta}\left(1+\log \pi_g-\log \pi_g^t\right)+\nu=0 .
    \end{equation}
    Define $\tau:=\frac{1}{\eta}-\lambda >0$. Simplifying the stationarity condition gives:
    \begin{equation}
      \tau \log \pi_g=u_g^t+\frac{1}{\eta} \log \pi_g^t+B,
    \end{equation}
    where $B$ is a constant. Exponentiating and applying the normalization condition $\sum_g \pi_g=1$ yields
    \begin{equation}\label{eq: sovle}
      \pi_g^{t+1}=\frac{\left(\pi_g^t\right)^{\frac{1}{1-\eta \lambda}} \exp \left(\frac{\eta}{1-\eta \lambda} u_g^t\right)}{\sum_j\left(\pi_j^t\right)^{\frac{1}{1-\eta \lambda}} \exp \left(\frac{\eta}{1-\eta \lambda} u_j^t\right)}, \quad 0<\eta\lambda<1,
    \end{equation}
    Equivalently,
    \begin{equation}
      \pi_g^{t+1} \propto\left(\pi_g^t\right)^{1 /(1-\eta \lambda)} \cdot \exp \left(u_g^t / \tau\right), \quad \tau=\frac{1-\eta \lambda}{\eta}.
    \end{equation}
    The above optimal solution can be interpreted as \textbf{a softmax distribution with a base distribution $\boldsymbol{\pi}^t$ and a temperature parameter $\tau=(1-\eta\lambda)/\eta$}. Therefore, as $\lambda$ increases, $\tau$ decreases, and the routing distribution becomes sharper.
  \end{proof}
  
  \begin{definition}[$\epsilon$-soft top-$k$]
    \emph{Given a routing distribution $\boldsymbol{\pi}=[\pi_1,\cdots,\pi_m]\in\mathbb{R}^m$ and gain scores $\boldsymbol{u}=[u_1,\cdots,u_m]\in\mathbb{R}^m$. Let $\operatorname{Top}_k(\boldsymbol{u})$ be the index set of the top-$k$ elements in $\boldsymbol{u}$. If}
  \begin{equation}
    \sum_{i \notin \operatorname{Top}_k(\boldsymbol{u})} \pi_i \leq \epsilon,
  \end{equation}
  \emph{then $\boldsymbol{\pi}$ is said to be an $\epsilon$-soft top-$k$ with respect to $\boldsymbol{u}$.}
  \end{definition}
  
  \begin{corollary}[$\epsilon$-soft top-$k$ Approximation]
    \emph{There exists a threshold $\theta$ such that when $\lambda\geq \theta$, the entropy-driven routing mechanism approximates an $\epsilon$-soft top-$k$ routing.}
  \end{corollary}
  \begin{proof}
    Let $u_1\geq \cdots \geq u_m$ be arranged in descending order, and define $\delta_k:=u_k - u_{k+1} >0 \ (1\leq k\leq m)$. From Eq.~(\ref{eq: sovle}), the ratio satisfies
  \begin{equation}
    \frac{\pi_{j}^{t+1}}{\pi_{k}^{t+1}} \leq\left(\frac{\pi_{j}^t}{\pi_{k}^t}\right)^{\frac{1}{1-\eta \lambda}} \exp \left(-\frac{\eta}{1-\eta \lambda} \delta_k\right) \quad (j>k).
  \end{equation}
  To achieve an approximate $\epsilon$-soft top-$k$ routing scheme, the optimal routing weights should concentrate as much as possible on the top-$k$ experts, while the total weight assigned to the remaining $m-k$ experts should be minimized.
  Since $\pi_{1}^{t+1}\geq\cdots\geq\pi_{m}^{t+1}$, we have
  \begin{equation}
    \begin{gathered}
      \sum_{j>k} \pi_j^{t+1} \leq \frac{m-k}{k} \exp \left(-\frac{\eta}{1-\eta \lambda} \delta_k\right)
      \end{gathered}
  \end{equation}
  Suppose there exists a small constant $\epsilon\in(0,1)$. As long as $\sum_{j>k} \pi_j^{t+1}\leq \epsilon$, the optimal routing distribution allocates at least a fraction $1-\epsilon$ of the probability mass to the top-$k$ experts, thereby making the entropy-driven routing mechanism approximate an $\epsilon$-soft top-$k$ routing. Hence, requiring
  \begin{equation}
     \frac{m-k}{k}\exp \left(-\frac{\eta}{1-\eta \lambda} \delta_k\right) \leq\epsilon,
  \end{equation}
  yields the condition for approximation. For this inequality to provide a meaningful lower bound on the threshold, we need $\frac{k\epsilon}{m-k}<1$, yielding
  \begin{equation}
  \lambda\geq\frac{1}{\eta}+\frac{\delta_k}{\log\frac{k\epsilon}{m-k}}.
  \end{equation}
  When $\theta=1/\eta+\delta_k/\log\frac{k\epsilon}{m-k}$, the entropy-driven routing mechanism approximates an $\epsilon$-soft top-$k$ routing.
\end{proof}

\section{Evaluations}\label{sec: experiment}
\subsection{Experimental Settings}
\subsubsection{Datasets}
We conduct extensive experiments on 12 benchmark datasets, including: 
1) Six homophilous datasets: Computers, Photo~\cite{mcauley2015image}, Coauthor CS, Coauthor Physics~\cite{Shchur2018PitfallsOG}, Facebook~\cite{rozemberczki20201R} and Ogbn-arxiv~\cite{hu2020open}; and
2) Six heterophilous datasets: Actor~\cite{tang2009S}, Squirrel-fix, Chameleon-fix, Tolokers, Roman-empire~\cite{critical} and Penn94~\cite{NEURIPS2021_ae816a80}.
For most datasets we use random splitting (48\%:32\%:20\% for training, validation and testing). For Ogbn-arxiv, we use the public splits in OGB~\cite{hu2020open}.
Refer to Appendix~\ref{app:datasets} for more dataset details.

\subsubsection{Baselines}
% We compare our methods with a total of 15 competitors in three categories, including classical methods such as two-layer MLP, GCN~\cite{kipf2016semi}, SAGE~\cite{hamilton2017inductive} and GAT~\cite{velivckovic2017gat}, spectral methods such as H2GCN~\cite{zhu2020h2gcn}, GPRGNN~\cite{chien2020adaptive}, FAGCN~\cite{bo2021fagcn} and ACM-GCN~\cite{luan2022acm}, spatial methods such as CPGNN\cite{zhu2021cpgnn}, WRGAT~\cite{suresh2021wrgat}, GBKSage~\cite{du2022gbk}, HOG~\cite{wang2022hog}, GloGNN~\cite{li2022glognn}, CAGNN~\cite{cagnn} and SNGNN~\cite{zou2023similarity}. The methods are the latest and tailor-made for the heterophily problem. We provide more details for these competitors in Appendix~\ref{app:baselines}.
We compare \model with a total of 20 baselines in four categories, including:
1) Vanilla model: MLP, GCN~\cite{GCN}, GAT~\cite{GAT} and GraphSAGE~\cite{GraphSAGE}; 
2) Heterophilous GNNs: LINKX~\cite{NEURIPS2021_ae816a80}, H2GCN~\cite{H2GCN2020}, GPRGNN~\cite{chien2020adaptive}, FAGCN~\cite{fagcn2021}, ACMGCN~\cite{luan2022revisiting}, GloGNN~\cite{li2022glognn} and FSGNN~\cite{MAURYA2022101695}; 
3) Graph Transformer (GT): vanilla GT, ANS-GT~\cite{ASN-GT}, NAGphormer~\cite{chennagphormer}, SGFormer~\cite{SGFormer}, Exphormer~\cite{exphormer} and Difformer~\cite{wu2023difformer};
4) Graph MoE: GMoE~\cite{wang2023graph}, DAMoE~\cite{YAO2026108064} and NodeMoE~\cite{han2024node}.
Refer to Appendix~\ref{app:baselines} for more details.

\subsubsection{Setup}
% To ensure the stability and reproducibility of the results, we use ten consecutive seeds to fix the data splitting and model initialization. For all methods, we set the maximum epochs to 500 with 100 epochs patience for early stopping, the hidden dimension to 512, and the optimizer to Adam~\cite{kingma2014adam}. For common parameters like learning rate, weight decay and dropout rate, we search in the same parameter space if their codes are publicly available. 
% The key parameters of NCGCN/NCSAGE like hop $k$ and threshold $T$ will be searched in their respective parameter spaces $k\in\{1, 2\}$ and $T\in\{0.3,0.4,0.5,0.6,0.7\}$. 
% Detailed parameter settings are reported in Appendix~\ref{app:parameter}.
We utilize 10 random seeds to fix the data splits and model initialization, and report the average accuracy and standard deviation over 10 runs.
For all methods, we set the search space of common parameters as follows:
maximum epochs to 500 with 100 patience,
hidden dimension $d^\prime$ to 64,
optimizer to AdamW,
learning rate in \{0.005, 0.01, 0.05, 0.1\},
dropout rate in \{0.1, 0.3, 0.5, 0.7, 0.9\}.
For \model, the number of MoE-blocks in \{3,4,5,6\} is searched for Ogbn-arxiv while a fixed value of 2 is used for other datasets.
For all baselines, we search the optimal parameters in the same parameter spaces.
Moreover, \model are implemented in PyTorch 1.11.0, Pytorch-Geometric 2.1.0 with CUDA 12.0 and Python 3.9.
All experiments are conducted at NVIDIA A100 40GB.
Refer to Appendix~\ref{app:parameter} for more parameter settings.

\subsection{Analysis}
We answer seven \textbf{R}esearch \textbf{Q}uestions (\textbf{RQ}) and demonstrate our arguments by extended experiments. 
% For simplicity, we only discuss the validity of our framework on \model(GCN-like \textbf{P}) in \emph{RQ2} - \emph{RQ5}.}

\subsubsection{\textbf{RQ1: Does GNNMoE surpass baselines across homophilous and heterophilous datasets?}}
Table~\ref{tab: main} reports node classification performance on 12 benchmarks, along with local rankings on homophilous and heterophilous graphs and global rankings across all graphs. Overall, \model achieves the best performance on 11 out of the 12 datasets, and its three backbone variants consistently attain the top overall ranks on all graphs, demonstrating powerful cross-graph generalization.

% In terms of baselines, heterophilous GNNs and Graph MoE models generally outperform vanilla GNNs. Notably, Graph MoE surpasses several specialized heterophilous graph models on multiple datasets, suggesting that the mixture-of-experts (MoE) mechanism can adaptively handle diverse structural and attribute patterns, effectively capturing node representations. 
Regarding baseline, heterophilous GNNs, GTs, and Graph MoE generally outperform vanilla GNNs. In particular, H2GCN, FSGNN, and GMoE stand out within their respective methodological families, which underscores the effectiveness of heterophilous message passing and MoE mechanisms.
By contrast, graph transformers exhibit relatively moderate performance and frequently encounter out-of-memory (OOM) issues on large graphs. This limitation arises from the quadratic computational complexity and high memory consumption introduced by global self-attention. Moreover, global modeling may inject additional noise in graphs with moderate or low homophily, thereby diminishing discriminative power.
By comparison, \model consistently outperforms all baselines, achieving relative performance gains of 4.63\% $\sim$ 20.47\% over vanilla GNNs, 2.69\% $\sim$ 12.06\% over heterogeneous GNNs, 3.14\% $\sim$ 7.25\% over GTs, and 2.51\% $\sim$ 10.45\% over Graph MoE. Moreover, \model does not encounter OOM on large-scale graphs, reflecting excellent scalability.

\begin{table*}[!htb]
  \centering
  \caption{Node classification results: average test accuracy (\%) $\pm$ standard deviation. ``Local Rank'' indicates the average performance ranking across homophilous or heterophilous datasets, ``Global Rank'' indicates the average performance ranking across all datasets. Boldface letters mark the best performance while underlined letters indicate the second best.}
  \renewcommand\arraystretch{1.3}
  \label{tab: main}
  \resizebox{\textwidth}{!}{
    \begin{tabular}{c|c|ccccccc|ccccccc|c} 
    \hline\hline
    \multicolumn{2}{l|}{\diagbox[width=14em]{Method}{Dataset}}                           & Computers             & Photo                 & \begin{tabular}[c]{@{}c@{}}Coauthor\\CS\end{tabular} & \begin{tabular}[c]{@{}c@{}}Coauthor\\Physics\end{tabular} & Facebook              & Ogbn-arxiv            & \begin{tabular}[c]{@{}c@{}}local\\rank\end{tabular}     & Actor                 & \begin{tabular}[c]{@{}c@{}}Chameleon\\-fix\end{tabular} & Squirrel-fix          & Tolokers              & \begin{tabular}[c]{@{}c@{}}Roman\\-empire\end{tabular} & Penn94                & \begin{tabular}[c]{@{}c@{}}local\\rank\end{tabular}     & \begin{tabular}[c]{@{}c@{}}global\\rank\end{tabular}     \\ 
    \hline
    \multirow{4}{*}{Vanilla}                      & MLP         & 85.01 $\pm$ 0.84          & 92.00 $\pm$ 0.56          & 94.80 $\pm$ 0.35                                         & 96.11 $\pm$ 0.14                                              & 76.86 $\pm$ 0.34          & 53.46 $\pm$ 0.35          & 19.67          & 37.14 $\pm$ 1.06          & 33.31 $\pm$ 2.32                                            & 34.47 $\pm$ 3.09          & 53.18 $\pm$ 6.35          & 65.98 $\pm$ 0.43                                           & 75.18 $\pm$ 0.35          & 18.33          & 19.00           \\
                                                  & GCN         & 91.17 $\pm$ 0.54          & 94.26 $\pm$ 0.59          & 93.40 $\pm$ 0.45                                         & 96.37 $\pm$ 0.20                                              & 93.98 $\pm$ 0.34          & 69.71 $\pm$ 0.18          & 16.00          & 30.65 $\pm$ 1.06          & 41.85 $\pm$ 3.22                                            & 33.89 $\pm$ 2.61          & 70.34 $\pm$ 1.64          & 50.76 $\pm$ 0.46                                           & 80.45 $\pm$ 0.27          & 18.83          & 17.42           \\
                                                  & GAT         & 91.44 $\pm$ 0.43          & 94.42 $\pm$ 0.61          & 93.20 $\pm$ 0.64                                         & 96.28 $\pm$ 0.31                                              & 94.03 $\pm$ 0.36          & 70.03 $\pm$ 0.42          & 14.83          & 30.58 $\pm$ 1.18          & 43.31 $\pm$ 3.42                                            & 36.27 $\pm$ 2.12          & 79.93 $\pm$ 0.77          & 57.34 $\pm$ 1.81                                           & 78.10 $\pm$ 1.28          & 17.33          & 16.08           \\
                                                  & GraphSAGE   & 90.94 $\pm$ 0.56          & 95.41 $\pm$ 0.45          & 94.17 $\pm$ 0.46                                         & 96.69 $\pm$ 0.23                                              & 93.72 $\pm$ 0.35          & 69.15 $\pm$ 0.18          & 14.00          & 37.60 $\pm$ 0.95          & 44.94 $\pm$ 3.67                                            & 36.61 $\pm$ 3.06          & 82.37 $\pm$ 0.64          & 77.77 $\pm$ 0.49                                           & OOM                       & 11.33          & 12.67           \\ 
  \hdashline      
    \multirow{7}{*}{Hetero}                       & H2GCN       & 91.69 $\pm$ 0.33          & 95.59 $\pm$ 0.48          & 95.62 $\pm$ 0.27                                         & 97.00 $\pm$ 0.16                                              & 94.36 $\pm$ 0.32          & OOM                       & 7.33          & 37.27 $\pm$ 1.27          & 43.09 $\pm$ 3.85                                            & 40.07 $\pm$ 2.73          & 81.34 $\pm$ 1.16          & 79.47 $\pm$ 0.43                                           & 75.91 $\pm$ 0.44          & 11.83          & 9.58           \\
                                                  & GPRGNN      & 91.80 $\pm$ 0.55          & 95.44 $\pm$ 0.33          & 95.17 $\pm$ 0.34                                         & 96.94 $\pm$ 0.20                                              & 94.84 $\pm$ 0.24          & 69.95 $\pm$ 0.19          & 8.00          & 36.89 $\pm$ 0.83          & 44.27 $\pm$ 5.23                                            & 40.58 $\pm$ 2.00          & 73.84 $\pm$ 1.40          & 67.72 $\pm$ 0.63                                           & 84.34 $\pm$ 0.29          & 12.00          & 10.00           \\
                                                  & FAGCN       & 89.54 $\pm$ 0.75          & 94.44 $\pm$ 0.62          & 94.93 $\pm$ 0.22                                         & 96.91 $\pm$ 0.27                                              & 91.90 $\pm$ 1.95          & 66.87 $\pm$ 1.48          & 15.67         & 37.59 $\pm$ 0.95          & 45.28 $\pm$ 4.33                                            & 41.05 $\pm$ 2.67          & 81.38 $\pm$ 1.34          & 75.83 $\pm$ 0.35                                           & 79.01 $\pm$ 1.09          & 9.17           & 12.42           \\
                                                  & ACMGCN      & 91.66 $\pm$ 0.78          & 95.42 $\pm$ 0.39          & 95.47 $\pm$ 0.33                                         & 97.00 $\pm$ 0.27                                              & 94.27 $\pm$ 0.33          & 69.98 $\pm$ 0.11          & 7.83          & 36.89 $\pm$ 1.13          & 43.99 $\pm$ 2.02                                            & 36.58 $\pm$ 2.75          & 83.52 $\pm$ 0.87          & 81.57 $\pm$ 0.35                                           & 83.01 $\pm$ 0.46          & 11.50          & 9.67           \\
                                                  & GloGNN      & 89.48 $\pm$ 0.63          & 94.34 $\pm$ 0.58          & 95.32 $\pm$ 0.29                                         & OOM                                                           & 84.57 $\pm$ 0.62          & OOM                       & 18.50         & 37.30 $\pm$ 1.41          & 41.46 $\pm$ 3.89                                            & 37.66 $\pm$ 2.12          & 58.74 $\pm$ 13.41         & 66.46 $\pm$ 0.41                                           & 85.33 $\pm$ 0.27          & 14.33          & 16.42           \\
                                                  & FSGNN       & 91.03 $\pm$ 0.56          & 95.50 $\pm$ 0.41          & 95.51 $\pm$ 0.32                                         & 96.98 $\pm$ 0.20                                              & 94.32 $\pm$ 0.32          & 71.09 $\pm$ 0.21          & 8.17          & 37.14 $\pm$ 1.06          & 45.79 $\pm$ 3.31                                            & 38.25 $\pm$ 2.62          & 83.87 $\pm$ 0.98          & 79.76 $\pm$ 0.41                                           & 83.87 $\pm$ 0.98          & 8.33           & 8.25           \\
                                                  & LINKX       & 90.75 $\pm$ 0.36          & 94.58 $\pm$ 0.56          & 95.52 $\pm$ 0.30                                         & 96.93 $\pm$ 0.16                                              & 93.84 $\pm$ 0.32          & 66.16 $\pm$ 0.27          & 12.33         & 31.17 $\pm$ 0.61          & 44.94 $\pm$ 3.08                                            & 38.40 $\pm$ 3.54          & 77.55 $\pm$ 0.80          & 61.36 $\pm$ 0.60                                           & 84.97 $\pm$ 0.46          & 13.50          & 12.92           \\ 
  \hdashline      
    \multirow{6}{*}{GT}                           & Vanilla GT  & 84.41 $\pm$ 0.72          & 91.58 $\pm$ 0.73          & 94.61 $\pm$ 0.30                                         & OOM                                                           & OOM                       & OOM                       & 20.67          & 37.08 $\pm$ 1.08          & 44.27 $\pm$ 3.98                                            & 39.55 $\pm$ 3.10          & 72.24 $\pm$ 1.17          & OOM                                                        & OOM                       & 15.50          & 18.08           \\
                                                  & ANS-GT      & 90.01 $\pm$ 0.38          & 94.51 $\pm$ 0.24          & 93.93 $\pm$ 0.23                                         & 96.28 $\pm$ 0.19                                              & 92.61 $\pm$ 0.16          & OOM                       & 17.50          & \uline{37.80 $\pm$ 0.95}  & 40.74 $\pm$ 2.26                                            & 36.65 $\pm$ 0.80          & 76.91 $\pm$ 0.85          & 80.36 $\pm$ 0.71                                           & OOM                       & 13.67          & 15.58           \\
                                                  & NAGFormer   & 90.22 $\pm$ 0.42          & 94.95 $\pm$ 0.52          & 94.96 $\pm$ 0.25                                         & 96.43 $\pm$ 0.20                                              & 93.35 $\pm$ 0.28          & 70.25 $\pm$ 0.13          & 13.83          & 36.99 $\pm$ 1.39          & 46.12 $\pm$ 2.25                                            & 38.31 $\pm$ 2.43          & 66.73 $\pm$ 1.18          & 75.92 $\pm$ 0.69                                           & 73.98 $\pm$ 0.53          & 13.67          & 13.75           \\
                                                  & SGFormer    & 90.70 $\pm$ 0.59          & 94.46 $\pm$ 0.49          & 95.21 $\pm$ 0.20                                         & 96.87 $\pm$ 0.18                                              & 86.66 $\pm$ 0.54          & 65.84 $\pm$ 0.24          & 15.33          & 36.59 $\pm$ 0.90          & 44.27 $\pm$ 3.68                                            & 38.83 $\pm$ 2.19          & 80.46 $\pm$ 0.91          & 76.41 $\pm$ 0.50                                           & 76.65 $\pm$ 0.49          & 13.50          & 14.42           \\
                                                  & Exphormer   & 91.46 $\pm$ 0.51          & 95.42 $\pm$ 0.26          & 95.62 $\pm$ 0.29                                         & 96.89 $\pm$ 0.20                                              & 93.88 $\pm$ 0.40          & 71.59 $\pm$ 0.24          & 8.83           & 36.83 $\pm$ 1.10          & 42.58 $\pm$ 3.24                                            & 36.19 $\pm$ 3.20          & 82.26 $\pm$ 0.41          & \textbf{87.55 $\pm$ 1.13}                                  & OOM                       & 13.67          & 11.25           \\
                                                  & Difformer   & 91.52 $\pm$ 0.55          & 95.41 $\pm$ 0.38          & 95.49 $\pm$ 0.26                                         & 96.98 $\pm$ 0.22                                              & 94.23 $\pm$ 0.47          & OOM                       & 10.33          & 36.73 $\pm$ 1.27          & 44.44 $\pm$ 3.20                                           & 40.45 $\pm$ 2.51           & 81.04 $\pm$ 4.16          & 78.97 $\pm$ 0.54                                           & OOM                        & 12.83          & 11.58           \\ 
   \hdashline      
    \multicolumn{1}{c|}{\multirow{3}{*}{\begin{tabular}[c]{@{}c@{}}Graph\\MoE\end{tabular}}}                    & GMoE        & 91.37 $\pm$ 0.49          & 94.51 $\pm$ 0.68          & 93.18 $\pm$ 0.58                     & 96.48 $\pm$ 0.23   & 94.90 $\pm$ 0.25           & 71.88 $\pm$ 0.32           & 11.67         & 33.78 $\pm$ 1.32          & 46.69 $\pm$ 3.55                                            & \uline{42.24 $\pm$ 2.45}  & 85.21 $\pm$ 0.40          & 84.78 $\pm$ 0.76                                           & 79.03 $\pm$ 0.78          & 7.00           & 9.33           \\
                                                  & DAMoE       & 91.57 $\pm$ 0.64          & 94.39 $\pm$ 0.53          & 93.42 $\pm$ 0.50                                         & 96.42 $\pm$ 0.28                                              & 94.96 $\pm$ 0.21          & 71.76 $\pm$ 0.15          & 11.83          & 28.76 $\pm$ 1.01          & 45.51 $\pm$ 2.80                                            & 41.08 $\pm$ 2.08          & 51.45 $\pm$ 1.07          & 81.92 $\pm$ 0.52                                           & 78.04 $\pm$ 0.58          & 13.00          & 12.42           \\
                                                  & NodeMoE     & \uline{91.87 $\pm$ 0.33}  & 95.63 $\pm$ 0.41          & OOM                                                      & OOM                                                           & 94.84 $\pm$ 0.28          & OOM                       & 12.33          & 36.28 $\pm$ 1.39          & 45.67 $\pm$ 4.54                                            & 40.49 $\pm$ 2.01          & 74.31 $\pm$ 0.87          & OOM                                                        & OOM                       & 14.83          & 13.58           \\ 
   \hdashline     
    \multirow{3}{*}{\textbf{GNNMoE}}              & GCN-like P  & \textbf{91.99 $\pm$ 0.42} & \textbf{95.82 $\pm$ 0.43} & \textbf{95.88 $\pm$ 0.26}       & \textbf{97.20 $\pm$ 0.13}      & 95.12 $\pm$ 0.26           & \uline{72.31 $\pm$ 0.27}  & \textbf{1.50}                                                         & 37.60 $\pm$ 1.75          & \textbf{47.98 $\pm$ 2.82}                                   & \textbf{42.67 $\pm$ 2.28}  & \textbf{85.32 $\pm$ 0.62}  & 85.09 $\pm$ 0.73                                           & \uline{85.35 $\pm$ 0.33}         & \textbf{2.00} & \textbf{1.75}  \\
                                                  & SAGE-like P & \uline{91.87 $\pm$ 0.44}  & 95.73 $\pm$ 0.24          & 95.72 $\pm$ 0.23                & \uline{97.16 $\pm$ 0.16}       & \uline{95.28 $\pm$ 0.26}   & 71.83 $\pm$ 0.18          & 2.67                                                                  & \textbf{38.04 $\pm$ 0.99} & \uline{47.75 $\pm$ 2.79}                                    & 41.78 $\pm$ 2.39          & 83.86 $\pm$ 0.79          & 86.02 $\pm$ 0.51                                           & \textbf{85.46 $\pm$ 0.27}  & \uline{2.50}         & \uline{2.58}           \\
                                                  & GAT-like P  & 91.66 $\pm$ 0.55          & \uline{95.78 $\pm$ 0.37}  & \uline{95.84 $\pm$ 0.33}        & \uline{97.16 $\pm$ 0.17}       & \textbf{95.30 $\pm$ 0.22}  & \textbf{72.54 $\pm$ 0.23} & \uline{2.33}                                                          & 37.53 $\pm$ 1.00          & 46.69 $\pm$ 3.77                                            & 41.12 $\pm$ 2.23          & \uline{85.29 $\pm$ 0.80}   & \uline{87.34 $\pm$ 0.62}                                   & 85.35 $\pm$ 0.34                  & 3.17   & 2.75           \\ 
    \hline\hline
     \end{tabular}}
\end{table*}

\begin{table*}[!htb]
  \centering
  \caption{Ablation studies on the key designs of GNNMoE.}
  \renewcommand\arraystretch{1.3}
  \label{tab: ablation}
  \resizebox{\textwidth}{!}{
    \begin{tabular}{c|c|ccccccc|ccccccc|c} 
    \hline\hline
    \multicolumn{2}{l|}{\diagbox[width=14em]{Method}{Dataset}}                           & Computers             & Photo                 & \begin{tabular}[c]{@{}c@{}}Coauthor\\CS\end{tabular} & \begin{tabular}[c]{@{}c@{}}Coauthor\\Physics\end{tabular} & Facebook              & Ogbn-arxiv            & \begin{tabular}[c]{@{}c@{}}local\\rank\end{tabular}     & Actor                 & \begin{tabular}[c]{@{}c@{}}Chameleon\\-fix\end{tabular} & Squirrel-fix          & Tolokers              & \begin{tabular}[c]{@{}c@{}}Roman\\-empire\end{tabular} & Penn94                & \begin{tabular}[c]{@{}c@{}}local\\rank\end{tabular}     & \begin{tabular}[c]{@{}c@{}}global\\rank\end{tabular}     \\ 
    \hline   
    \multirow{3}{*}{\textbf{GNNMoE}}              & GCN-like P  & \uline{91.99 $\pm$ 0.42}  & \textbf{95.82 $\pm$ 0.43} & \textbf{95.88 $\pm$ 0.26}       & \textbf{97.20 $\pm$ 0.13}      & 95.12 $\pm$ 0.26          & 72.31 $\pm$ 0.27          & \textbf{2.67} & 37.60 $\pm$ 1.75          & \textbf{47.98 $\pm$ 2.82}                                   & \uline{42.67 $\pm$ 2.28}  & \uline{85.32 $\pm$ 0.62}  & 85.09 $\pm$ 0.73                                           & \uline{85.35 $\pm$ 0.33}         & \textbf{2.67} & \textbf{2.67}  \\
                                                  & SAGE-like P & 91.87 $\pm$ 0.44          & 95.73 $\pm$ 0.24          & 95.72 $\pm$ 0.23                & \uline{97.16 $\pm$ 0.16}       & 95.28 $\pm$ 0.26          & 71.83 $\pm$ 0.18          & 5.00          & \textbf{38.04 $\pm$ 0.99} & \uline{47.75 $\pm$ 2.79}                                    & 41.78 $\pm$ 2.39          & 83.86 $\pm$ 0.79          & 86.02 $\pm$ 0.51                                           & \textbf{85.46 $\pm$ 0.27}        & \uline{3.33}  & 4.17           \\
                                                  & GAT-like P  & 91.66 $\pm$ 0.55          & 95.78 $\pm$ 0.37          & \uline{95.84 $\pm$ 0.33}        & \uline{97.16 $\pm$ 0.17}       & \uline{95.30 $\pm$ 0.22}  & \textbf{72.48 $\pm$ 0.23} & \uline{3.33}  & 37.53 $\pm$ 1.00          & 46.69 $\pm$ 3.77                                            & 41.12 $\pm$ 2.23          & 85.29 $\pm$ 0.80          & \textbf{87.34 $\pm$ 0.62}                                  & \uline{85.35 $\pm$ 0.34}         & 4.33          & \uline{3.83}           \\ 
     \hdashline      
    w/o $\operatorname{SR}(\cdot)$               & GCN-like P  & 91.65 $\pm$ 0.36          & 95.78 $\pm$ 0.29          & 95.83 $\pm$ 0.23             & 97.12 $\pm$ 0.18                  & 95.23 $\pm$ 0.31          & 70.68 $\pm$ 0.22          & 6.33          & 37.38 $\pm$ 0.96          & 47.25 $\pm$ 2.80                                            & 41.37 $\pm$ 2.17          & 85.29 $\pm$ 0.83          & 83.88 $\pm$ 0.62                                           & 84.65 $\pm$ 0.35                 & 6.17          & 6.25           \\
    w/o EFFN                                      & GCN-like P  & 91.74 $\pm$ 0.47          & 95.37 $\pm$ 0.34          & 95.39 $\pm$ 0.35             & 96.86 $\pm$ 0.21                  & 95.29 $\pm$ 0.27          & 71.92 $\pm$ 0.23          & 7.83          & 33.72 $\pm$ 1.33          & 46.52 $\pm$ 3.13                                            & 40.92 $\pm$ 2.28          & 83.17 $\pm$ 1.63          & 82.41 $\pm$ 0.38                                           & 84.04 $\pm$ 1.16                 & 9.33          & 8.58           \\
    w/o $\operatorname{HR}(\cdot)$               & SwishGLU    & 91.76 $\pm$ 0.30          & 95.31 $\pm$ 0.38          & 95.80 $\pm$ 0.24             & 97.13 $\pm$ 0.19                  & 94.98 $\pm$ 0.27          & 70.15 $\pm$ 0.42          & 8.17          & 35.35 $\pm$ 1.05          & 46.91 $\pm$ 3.99                                            & 42.27 $\pm$ 2.25          & 81.53 $\pm$ 0.97          & 77.84 $\pm$ 1.04                                           & 85.27 $\pm$ 0.47                 & 7.50          & 7.83           \\
    w/o ARes                                      & GCN-like P  & 91.85 $\pm$ 0.45          & 95.67 $\pm$ 0.36          & 95.73 $\pm$ 0.31             & 97.07 $\pm$ 0.26                  & 94.57 $\pm$ 0.41          & 71.76 $\pm$ 0.13          & 7.50          & 37.51 $\pm$ 0.99          & 45.15 $\pm$ 1.65                                            & 40.79 $\pm$ 3.10          & 84.38 $\pm$ 0.81          & 84.14 $\pm$ 0.96                                           & 79.62 $\pm$ 0.58                 & 8.67          & 8.08           \\ 
    \hdashline            
    \multirow{3}{*}{w/o $\mathcal{L}_\text{route}$}& GCN-like P  & \textbf{92.17 $\pm$ 0.50} & \uline{95.81 $\pm$ 0.41}  & 95.81 $\pm$ 0.26             & 97.03 $\pm$ 0.13                  & \textbf{95.53 $\pm$ 0.35}  & 72.29 $\pm$ 0.16         & 3.50          & 37.59 $\pm$ 1.36          & 47.19 $\pm$ 2.93                                            & \textbf{44.02 $\pm$ 2.59} & 84.77 $\pm$ 0.93          & 85.05 $\pm$ 0.55                                           & 84.61 $\pm$ 0.39                 & 4.67           & 4.08   \\
                                                  & SAGE-like P & 91.85 $\pm$ 0.39          & 95.46 $\pm$ 0.24          & 95.68 $\pm$ 0.24             & 96.81 $\pm$ 0.22                  & 94.63 $\pm$ 0.36          & 71.94 $\pm$ 0.25          & 8.17          & \uline{37.97 $\pm$ 1.01}  & 45.73 $\pm$ 3.19                                            & 39.19 $\pm$ 2.84          & 83.96 $\pm$ 0.75          & 86.00 $\pm$ 0.45                                           & 84.05 $\pm$ 0.37                 & 6.67           & 7.42          \\
                                                  & GAT-like P  & 91.98 $\pm$ 0.46          & 95.71 $\pm$ 0.37          & 95.72 $\pm$ 0.23             & 97.05 $\pm$ 0.19                  & 95.21 $\pm$ 0.25          & \uline{72.45 $\pm$ 0.32}  & 5.50          & 37.76 $\pm$ 0.98          & 45.56 $\pm$ 3.94                                            & 39.19 $\pm$ 2.38          & \textbf{85.45 $\pm$ 0.94} & \uline{87.29 $\pm$ 0.60}                                           & 81.98 $\pm$ 0.47         & 6.00           & 5.75           \\
    \hdashline          
    $\Delta \tau$                                       & GCN-like P  & 91.67 $\pm$ 0.50          & 95.37 $\pm$ 0.52          & 95.81 $\pm$ 0.24             & 97.15 $\pm$ 0.19                  & 95.21 $\pm$ 0.33          & 71.31 $\pm$ 0.11          & 6.83          & 36.80 $\pm$ 0.94          & 47.70 $\pm$ 3.44                                            & 42.20 $\pm$ 2.41          & 83.58 $\pm$ 1.03          & 84.94 $\pm$ 0.99                                           & 85.19 $\pm$ 0.37                 & 6.17          & 6.50           \\
    \hline\hline
    \end{tabular}}
\end{table*}

\subsubsection{\textbf{RQ2: Do the key designs in GNNMoE work?}}
The \model framework incorporates several key components, including a soft routing layer, an enhanced FFN with a hard routing layer (EFFN), an entropy-driven routing adapter and adaptive residual connections.
To further investigate the effectiveness of these designs, we conduct extensive ablation studies, as summarized in Table~\ref{tab: ablation}:
(1)
We first remove the soft routing layer in MoE-blocks and replace routing with simple mean aggregation over expert outputs (w/o $\operatorname{SR}(\cdot)$). We observe that the overall performance of the ablation model degrades by 0.46\% and 1.26\% on homophilous and heterophilous datasets respectively, indicating that expert selection via soft routing effectively modulates message-passing strategies and enhances encoding capacity across different scenarios.
(2)
We then remove the hard routing layer in EFFN and directly adopt SwishGLU activation as the FFN (w/o $\operatorname{HR}(\cdot)$). In this case, the overall performance of the ablation model drops by 0.69\% and 3.94\% on homophilous and heterophilous datasets respectively, suggesting that the expert routing within the hard routing layer effectively adapts activation patterns and strengthens expressiveness, particularly on heterophilous graphs. Furthermore, when we entirely remove EFFN (w/o EFFN), the additional decline in overall performance further underscores the contribution of EFFN to the expressive power of \model.
(3) 
Removing adaptive residual connections (w/o ARes) decreases overall performance, particularly on heterophilous graphs by 3.43\%, indicating that initial features enhance the discriminability of \model and that heterophilous nodes depend more strongly on their intrinsic attributes.
(4)
We further remove the entropy-driven routing adapter (w/o $\mathcal{L}_\text{route}$). Across different propagation operators, the ablation models consistently exhibit lower overall rankings compared to their corresponding full models, demonstrating that the entropy-driven adaptive routing mechanism effectively improves the cross-dataset generalization of \model.
(5) 
Finally, instead of introducing the entropy-driven routing adapter, we reshape the routing weight distribution by adjusting the temperature of the softmax in Eq.~(\ref{eq: pi}) ($\Delta \tau$). This alternative underperforms \model across all datasets, highlighting the superiority of our routing entropy regularization mechanism.

\subsubsection{\textbf{RQ3: How does the routing entropy regularization enhance the generalization of GNNMoE?}}
To further understand how the routing entropy regularization endows \model with flexibility and generalization when representing different nodes, we visualize the averaged expert-routing weight distributions assigned to all nodes on each dataset by the optimally tuned \model(GCN-like P) and its ablation model (w/o $\mathcal{L}_\text{route}$), as shown in Fig.~\ref{fig:weight-dis}. Note that only the weight distributions of the first two MoE-blocks are visualized. The observations are as follows:
(1) Without introducing the routing entropy regularization (w/o $\mathcal{L}_\text{route}$), the model already exhibits differentiated and structured routing weight distributions across different datasets. Specifically, in homophilous graphs, the weight distributions are sharper, with a few experts, particularly the propagation-related ones (\textsf{PP}), playing dominant roles, which indicates that propagation operations help optimize node representations in homophilous scenarios. Conversely, in heterophilous graphs, the weight distributions are relatively smooth, indicating that multiple experts collaborate and jointly contribute to the representation learning process.
(2) After introducing the routing entropy regularization, the routing weight distributions undergo varying degrees of adjustment. For example, in homophilous graphs such as Computers and Photo, the distributions become sharper, particularly in Block 2, where the weights are more concentrated. Conversely, in homophilous graphs like Coauthor CS and Coauthor Physics, the distributions in Block 1 become smoother. This demonstrates that the entropy-driven routing regularization can break conventional patterns (\eg overly sharp distributions in homophilous graphs) and promote more flexible expert-collaboration modes tailored to specific datasets, thereby improving GNNMoE's generalization. 
In heterophilous graphs, the regularization subtly refines the routing distribution of Block 1 (keeping it relatively smooth) while sharpening that of Block 2. The former highlights the necessity of multi-expert cooperation for heterophilous representation learning, whereas the latter indicates that eliminating suboptimal experts after collaboration helps prevent noise.

% \blue{These empirical findings further support our theoretical analysis in Section~\ref{sec:method}, which states that entropy regularization helps suppress suboptimal experts and facilitates more efficient expert selection.}
\begin{figure*}[!htb]
    \centering
    \includegraphics[width =\linewidth]{Fig5-weight-gcn.pdf}
    \caption{Visualization of routing weight distributions before and after introducing the routing entropy regularization mechanism (P: GCN-like P).}
    \label{fig:weight-dis}
\end{figure*}
\begin{figure*}[!htb]
    \centering
    \includegraphics[width =\linewidth]{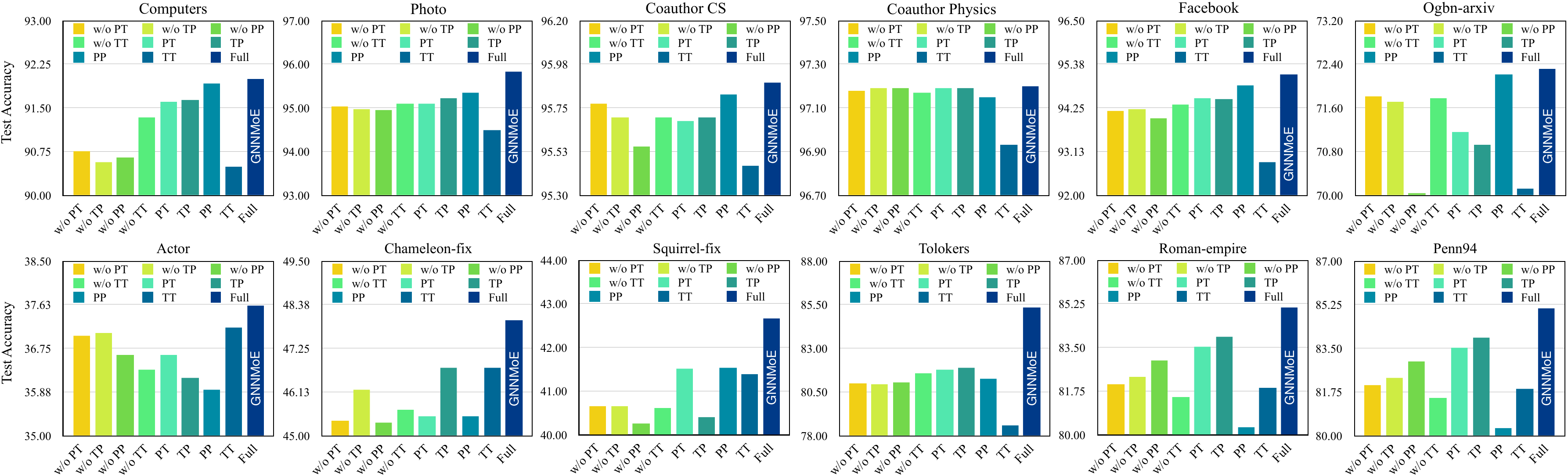}
    \caption{The impact of expert diversity on the performance of \model.}
    \label{fig: expert-diversity}
\end{figure*}

\subsubsection{\textbf{RQ4: How does expert diversity influence model performance?}}
In the expert network of \model, we design four distinct message-passing expert to adapt to diverse neighborhood structures and encoding requirements. To further investigate the impact of individual experts and expert diversity on model performance, we conduct several experiments: (1) removing any single expert from \model while keeping the remaining three; and (2) constructing an expert network using four identical experts. Fig.~\ref{fig: expert-diversity} summarizes the performance comparison across all datasets, from which several observations can be made:
(1)
Removing any expert or using only a single expert leads to performance degradation, with a more significant drop observed when key experts are removed. This indicates that each expert learns complementary latent features and collaborates effectively through adaptive routing;
(2)
On homophilous graphs, using only the key experts (with the same total number) can achieve performance close to that of \model, highlighting the preference for specialized encoding patterns in homophilous scenarios;
(3)
On heterophilous graphs, all variant models exhibit a more substantial performance drop, implying that the model relies more heavily on expert diversity to realize complementary collaboration among encoding patterns;
(4)
When expert diversity is limited, increasing the number of experts can yield comparable performance gains, as diversity in this case arises from parameter variations within otherwise identical experts.
Overall, by leveraging a diverse expert network coupled with a context-adaptive expert routing mechanism, \model demonstrates strong generalization across datasets.

\subsubsection{\textbf{RQ5: How do the entropy-driven routing mechanisms affect model performance?}}
To evaluate the effectiveness of the entropy-driven routing mechanism, we design several comparative experiments: (1) directly averaging the outputs of all experts (Mean); (2) aggregating information from the top-$k$ experts with the highest routing weights (Top-k); and (3) generating expert routing weights through learnable parameters, which can be viewed as a form of dot-product attention (Dot-Att). The results, summarized in Table~\ref{tab:gate}, show that our entropy-driven routing mechanism achieves consistently superior performance across most datasets, surpassing all alternative routing strategies. Specifically:
(1) Compared with the dot-product attention routing, \model avoids overfitting during node-level expert aggregation, thereby enhancing generalization;
(2) Compared with the mean routing, the entropy-driven routing mechanism enables adaptive expert selection, effectively reducing noise interference;
(3) Compared with the Top-k routing, the entropy-driven routing mechanism requires no pre-defined hyperparameter $k$. Instead, it adaptively adjusts the routing weight distribution to emulate soft top-$k$ routing, significantly lowering the parameter scale while maintaining powerful performance.
In summary, the entropy-driven routing mechanism offers substantial advantages in improving model generalization, robustness, and efficiency.

\begin{table*}
  \centering
  \caption{Impact of routing mechanism on the performance of \model.}
  \label{tab:gate}
  \renewcommand\arraystretch{1.2}
  \resizebox{\textwidth}{!}{
    \begin{tabular}{c|c|ccccccc|ccccccc|c} 
    \hline\hline
    \multicolumn{2}{l|}{\diagbox[width=14em]{Method}{Dataset}}                                                     & Computers             & Photo                 & \begin{tabular}[c]{@{}c@{}}Coauthor\\CS\end{tabular} & \begin{tabular}[c]{@{}c@{}}Coauthor\\Physics\end{tabular} & Facebook              & Ogbn-arxiv               & \begin{tabular}[c]{@{}c@{}}local\\rank\end{tabular}           & Actor                 & \begin{tabular}[c]{@{}c@{}}Chameleon\\-fix\end{tabular} & Squirrel-fix          & Tolokers              & \begin{tabular}[c]{@{}c@{}}Roman\\-empire\end{tabular} & Penn94                 & \begin{tabular}[c]{@{}c@{}}local\\rank\end{tabular}  & \begin{tabular}[c]{@{}c@{}}global\\rank\end{tabular}\\ 
    \hline                                                    
    \multirow{3}{*}{\begin{tabular}[c]{@{}c@{}}GNNMoE\\Dot-Att\end{tabular}} & GCN-like P  & 90.92 ± 0.52          & 94.73 ± 0.53          & 95.46 ± 0.30                                         & 96.90 ± 0.16                                              & 93.93 ± 0.36          & 70.13 ± 0.17             & 10.83                                                         & 36.70 ± 0.82          & 46.24 ± 4.36                                            & 40.92 ± 2.87          & 73.79 ± 1.13          & 65.95 ± 0.46                                           & 84.54 ± 0.38           & 9.33 & 10.08 \\
                                                                                         & SAGE-like P & 90.98 ± 0.52          & 94.73 ± 0.41          & 95.46 ± 0.28                                         & 96.91 ± 0.16                                              & 93.86 ± 0.33          & 69.25 ± 0.26             & 11.00                                                         & 36.83 ± 0.67          & 46.12 ± 3.07                                            & 41.46 ± 2.14          & 73.62 ± 1.05          & 65.89 ± 0.44                                           & 84.59 ± 0.43           & 8.67 & 9.83 \\
                                                                                         & GAT-like P  & 90.93 ± 0.36          & 94.75 ± 0.41          & 95.49 ± 0.32                                         & 96.91 ± 0.17                                              & 93.89 ± 0.27          & 70.32 ± 0.31             & 10.00                                                         & 36.41 ± 1.37          & 46.12 ± 4.39                                            & 41.46 ± 2.49          & 73.42 ± 0.83          & 65.98 ± 0.39                                           & 84.50 ± 0.60           & 9.17 & 9.58 \\ 
    \hline                                                            
    \multirow{3}{*}{\begin{tabular}[c]{@{}c@{}}GNNMoE\\Mean\end{tabular}}                & GCN-like P  & \uline{92.15 ± 0.36}  & 95.78 ± 0.29          & 95.83 ± 0.23                                         & 97.12 ± 0.18                                              & 95.23 ± 0.31          & 70.68 ± 0.22             & 4.00                                                          & 37.38 ± 0.96          & 47.25 ± 2.80                                            & 41.37 ± 2.17          & 85.29 ± 0.83          & 83.88 ± 0.62                                           & 84.65 ± 0.35           & 5.17 & 4.58 \\
                                                                                         & SAGE-like P & 91.85 ± 0.44          & 95.61 ± 0.48          & 95.53 ± 0.25                                         & 96.96 ± 0.24                                              & 94.92 ± 0.22          & 69.74 ± 0.27             & 8.67                                                          & \uline{37.87 ± 1.27}  & 45.17 ± 4.31                                            & 39.03 ± 2.16          & 83.76 ± 1.14          & 85.63 ± 0.63                                           & 83.91 ± 0.35           & 8.00 & 8.33 \\
                                                                                         & GAT-like P  & 91.53 ± 0.52          & 95.72 ± 0.46          & 95.74 ± 0.31                                         & 97.03 ± 0.22                                              & 95.08 ± 0.37          & 71.52 ± 0.15             & 7.17                                                          & 37.33 ± 1.14          & 44.66 ± 3.25                                            & 39.48 ± 2.45          & 85.24 ± 0.75          & 85.32 ± 0.61                                           & 83.16 ± 0.52           & 8.33 & 7.75 \\ 
    \hline                                                            
    \multirow{3}{*}{\begin{tabular}[c]{@{}c@{}}GNNMoE\\Top-k\end{tabular}}               & GCN-like P  & \textbf{92.15 ± 0.35} & 95.78 ± 0.27          & 95.82 ± 0.29                                         & 97.18 ± 0.12                                              & 95.09 ± 0.27          & 70.57 ± 0.22             & 4.17                                                          & 37.28 ± 1.36          & 47.08 ± 3.72                                            & 41.55 ± 2.46          & 84.80 ± 0.70          & 84.72 ± 0.69                                           & 84.81 ± 0.33           & 5.50 & 4.83 \\
                                                                                         & SAGE-like P & 91.80 ± 0.46          & 95.71 ± 0.32          & 95.64 ± 0.26                                         & 97.12 ± 0.17                                              & 95.20 ± 0.26          & 69.48 ± 0.17             & 7.50                                                          & 37.49 ± 1.00          & 46.46 ± 3.24                                            & 40.43 ± 2.55          & 83.89 ± 0.71          & 85.20 ± 0.85                                           & 84.07 ± 0.39           & 7.33 & 7.42 \\
                                                                                         & GAT-like P  & 91.87 ± 0.33          & \uline{95.81 ± 0.46}  & \uline{95.84 ± 0.21}                                 & \uline{97.13 ± 0.17}                                      & 95.26 ± 0.25          & 71.61 ± 0.15             & 3.33                                                          & 37.05 ± 0.95          & 46.57 ± 2.40                                            & 41.06 ± 3.17          & 85.23 ± 0.79          & 85.14 ± 0.62                                           & 84.19 ± 0.27           & 7.17 & 5.25 \\ 
    \hline                                                            
    \multirow{3}{*}{GNNMoE}                                                              & GCN-like P  & 91.99 ± 0.42          & \textbf{95.82 ± 0.43} & \textbf{95.88 ± 0.26}                                & \textbf{97.20 ± 0.13}                                     & 95.12 ± 0.26          & \uline{72.31 ± 0.27}     & \textbf{2.33}                                                 & 37.60 ± 1.75          & \textbf{47.98 ± 2.82}                                   & \textbf{42.67 ± 2.28} & \textbf{85.32 ± 0.62} & 85.09 ± 0.73                                           & \uline{85.35 ± 0.33}   & \textbf{2.50} & \textbf{2.42} \\
                                                                                         & SAGE-like P & 91.87 ± 0.44          & 95.73 ± 0.24          & 95.72 ± 0.23                                         & 97.16 ± 0.16                                              & \uline{95.28 ± 0.26}  & 71.83 ± 0.18             & 4.17                                                          & \textbf{38.04 ± 0.99} & \uline{47.75 ± 2.79}                                    & \uline{41.78 ± 2.39}  & 83.86 ± 0.79          & \uline{86.02 ± 0.51}                                   & \textbf{85.46 ± 0.27}  & \uline{2.67}  & 3.42 \\
                                                                                         & GAT-like P  & 91.66 ± 0.55          & 95.78 ± 0.37          & 95.84 ± 0.33                                         & 97.16 ± 0.17                                              & \textbf{95.30 ± 0.22} & \textbf{72.48 ± 0.23}    & \uline{3.00}                                                  & 37.53 ± 1.00          & 46.69 ± 3.77                                            & 41.12 ± 2.23          & \uline{85.29 ± 0.80}  & \textbf{87.34 ± 0.62}                                  & 85.35 ± 0.34           & 3.50          & \uline{3.25} \\
    \hline
    \hline
    \end{tabular}
  }
\end{table*}

\begin{figure*}[!ht]
    \centering
    \includegraphics[width = \linewidth]{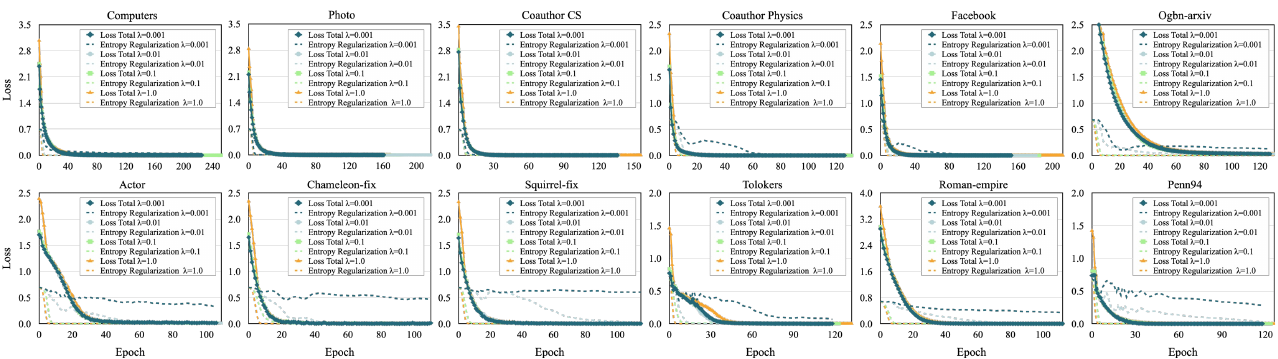}
    \caption{Visualization of training loss curves under different regularization coefficients $\lambda$.}
    \label{fig: loss}
\end{figure*}

\begin{figure}[htbp]
  \centering
  \begin{minipage}[b]{0.45\linewidth}
      \centering
      \includegraphics[width=\textwidth, page=1]{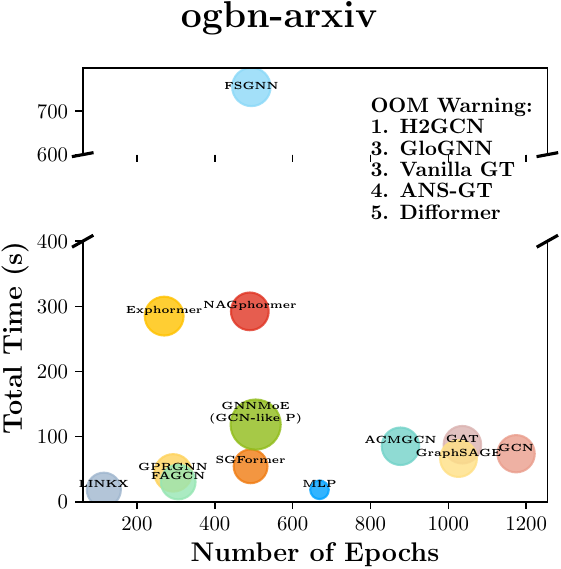}
  \end{minipage}
  \hspace{0.2cm}
  \begin{minipage}[b]{0.45\linewidth}
      \centering
      \includegraphics[width=\textwidth, page=1]{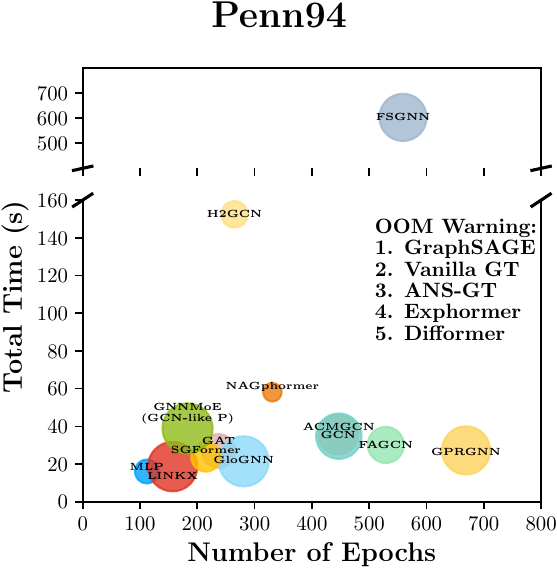}
  \end{minipage}
  \caption{Efficiency analysis on Ogbn-arxiv and Penn94.}
  \label{fig: time}
\end{figure}

\subsubsection{{\textbf{RQ6: Can GNNMoE, which incorporates both soft and hard routing mechanisms, be trained stably?}}}\label{sec: stab-training}
To further investigate the training stability of \model, we plot the training loss curves across all datasets, as shown in Fig.~\ref{fig: loss}. The results show that the loss decreases rapidly and converges steadily in the early training stage, with no evident oscillations or convergence bottlenecks throughout the process, highlighting the stability of model training.
Specifically, the soft routing mechanism in \model facilitates sufficient information fusion among experts in the early stages of training, thereby enhancing the model's expressive power. Meanwhile, the hard routing mechanism reinforces the determinism of expert activation, effectively suppressing redundant activations and noise interference. As training progresses, the routing entropy regularization gradually guides the adjustment of routing weight distributions. Notably, in homophilous graphs, regardless of the value of $\lambda$, the regularization loss converges toward zero, sharpening the routing distribution and realizing soft top-$k$ routing. In heterophilous graphs, when $\lambda$ is relatively small, the regularization loss converges to a nonzero value, maintaining smoother routing distributions and realizing soft weighting routing.
Overall, during training, the soft routing enables continuous optimization, the routing entropy regularization refines the allocation of routing weights, and the hard routing ensures precise expert selection and effective noise suppression. The synergy among these components allows \model to achieve more stable training while maintaining adaptive representational capability, ultimately driving continuous improvements in overall performance.

% \subsubsection{Efficiency Analysis}
\subsubsection{{\textbf{RQ7: How efficient is GNNMoE?}}}
Fig.~\ref{fig: time} illustrates the training efficiency and accuracy of \model compared with several representative methods on two large-scale datasets, Ogbn-arxiv and Penn94. The x-axis denotes the number of training epochs required to trigger early stopping, while the y-axis represents the total training time. The bubble size reflects model performance.
As shown, compared with spatial-domain GNNs such as FSGNN and several GT methods, \model achieves leading performance while reducing training time by approximately 2 to 7 times. Moreover, relative to vanilla GNNs and most spatial-domain GNNs, \model requires fewer training epochs, exhibiting a much faster convergence rate.
Overall, \model achieves high performance while maintaining excellent computational efficiency.

\section{Related Work}
\subsection{Shared-weight Paradigm}
The shared-weight paradigm represents the dominant design philosophy in GNNs, founded on the assumption that all nodes across a graph share an identical encoding mechanism. This paradigm performs end-to-end learning through globally unified feature transformation and neighborhood aggregation parameters, exemplified by the classic evolution from GCN to GraphSAGE and GAT. In spectral methods, GPRGNN achieves multi-scale feature aggregation through shared and learnable PageRank weights, while ACMGCN and FAGCN employ shared filter kernels across multiple channels to capture structural diversity. In spatial methods, approaches such as LINKX and GloGNN incorporate global structure or node correlation information, yet their aggregation weights remain globally shared. These GNNs indirectly adjust aggregation weights to accommodate structural variations, while the encoding function remains uniform across the entire graph. Recent graph transformer models, such as ANS-GT, NAGFormer, SGFormer, Exphormer, and Difformer, replace local convolution with global attention to capture long-range dependencies. However, their attention mechanisms are still uniformly parameterized, limiting the ability to distinguish heterophilous node features and often leading to noise accumulation in low-homophily scenarios. Overall, the shared-weight paradigm offers advantages in parameter efficiency and training stability, yet lacks the adaptivity required to capture structural and semantic diversity at the node level.

\subsection{Separated-Weight Paradigm}
The separated-weight paradigm aims to break the constraints of global parameter sharing by assigning distinct encoding weights to different subspaces or nodes according to their structural and semantic differences, thereby enabling finer-grained adaptive representation. Representative models such as NCGNN~\cite{NCGNN} quantify node heterophily via neighborhood confusion metric and partition nodes into distinct convolutional channels for coarse-grained separated encoding. GMoE treats multi-scale convolutions as parallel experts, each with independent parameters, and employs auxiliary regularization to ensure expert diversity, achieving structure-aware fine-grained separation. 
% MOWST constructs a hierarchical expert architecture that decouples lightweight MLPs and GNNs, dynamically allocating computation through a confidence-based mechanism, thus embodying a partially learnable separation strategy. 
These approaches exhibit clear advantages in parameter independence and node-specific modeling over shared-weight designs, but most still rely on coarse-grained grouping or static routing. Consequently, achieving a principled balance between parameter sharing and separation through a dynamic mechanism for adaptive allocation has become a key direction in graph representation learning, forming the central motivation for the proposed \model framework.

\section{Conclusions and Outlook}
This work explored the challenge of adaptive node representation under varying degrees of homophily and heterophily and proposed \model, an entropy-driven mixture of message-passing experts framework that unifies soft and hard routing in a flexible, learnable manner. The framework advances GNNs by enabling each node to select an optimal encoding path, achieving personalized representation and improved robustness across various graphs. Beyond its empirical superiority, \model offers a conceptual bridge between message-passing and expert-routing paradigms, suggesting a new direction for scalable, interpretable graph learning. Future research may extend this framework to more complex graphs and explore theoretical bounds of entropy-driven routing. Together, these findings point toward a broader principle: adaptivity and specialization are not competing goals but complementary forces for generalizable graph intelligence.

\bibliographystyle{IEEEtran}
\bibliography{Reference,IEEEabrv}

\newpage
% \mbox{}
% \newpage

\appendices
% 重定义附录二级标题
\renewcommand{\thesubsection}{\thesection-\arabic{subsection}}
\renewcommand{\thesubsectiondis}{\thesection-\arabic{subsection}}
\renewcommand{\thesubsubsection}{\thesubsection-\arabic{subsubsection}}
\renewcommand{\thesubsubsectiondis}{\thesubsection-\arabic{subsubsection}}

\setcounter{theorem}{0}% ——在附录开头清零
\setcounter{corollary}{0}
% 如果 lemma/corollary 等是独立计数器，也在这里各自清零

\section{Theoretical Proof}\label{app: proof}
\subsection{Theorem 1}

\begin{theorem}[\textbf{Temperature property of entropy-driven routing}]
  \emph{Suppose there are $m\ge 2$ message-passing experts and the routing weight distribution over experts is $\boldsymbol{\pi}=[\pi_1,\ldots,\pi_m]\in\Delta^{m}$, where feasible region
  $\Delta^{m}:=\{\boldsymbol{\pi}\in\mathbb{R}^m, \pi_g\ge 0, \sum_{g=1}^m \pi_g=1\}$.
  For a node $v_i$ at a given MoE-block, each expert has an instantaneous gain $u_g\in\mathbb{R}$ and $\boldsymbol{u}=[u_1, \ldots, u_m] \in \mathbb{R}^m$. Consider the following routing optimization problrm:} 
  \begin{equation}
    \min _{\boldsymbol{\pi} \in \Delta^m}\langle -\boldsymbol{u}, \boldsymbol{\pi}\rangle+\lambda H(\boldsymbol{\pi}) ,
  \end{equation}
  \emph{where $\lambda>0$ is the entropy regularization coefficient. Then the optimal routing $\pi_g^{t+1}(\lambda)$ has a unique closed-form solution.}
\end{theorem}
\begin{proof}
  To ensure consistency between the surrogate objective and the original objective, while maintaining descent during optimization, we further introduce a trust-region constraint based on KL divergence. At the $t$-th iteration, with step size $\eta>0$, the optimal routing is obtained by solving the following optimization problem:
  \begin{equation}
    \boldsymbol{\pi}^{t+1}=\argmin_{\boldsymbol{\pi} \in \Delta^m}\left\langle -\boldsymbol{u}^t, \boldsymbol{\pi}\right\rangle+\lambda H(\boldsymbol{\pi})+\frac{1}{\eta} \mathrm{KL}\left(\boldsymbol{\pi} \| \boldsymbol{\pi}^t\right).
  \end{equation}
  Write the optimization objective as:
  \begin{equation}
    \begin{aligned}
      J(\boldsymbol{\pi})=-\left\langle \boldsymbol{u}^t, \boldsymbol{\pi}\right\rangle&-\lambda \sum_g {\pi}_g \log {\pi}_g+\frac{1}{\eta} \sum_g {\pi}_g \log \frac{{\pi}_g}{\pi_g^t},\\
      &\sum_g \pi_g=1, \ \pi_g \geq 0.
    \end{aligned}
  \end{equation}
  To establish the existence of a unique minimizer, it is equivalent to proving that the Hessian matrix of $J(\boldsymbol{\pi})$ is positive definite in the interior of the feasible region. The Hessian matrix can be written as:
  % \begin{equation}
  %   \frac{\partial^2 {J}}{\partial \pi_g^2} =\left(\frac{1}{\eta}-\lambda \right)\cdot \frac{1}{\pi_g}, \qquad \frac{\partial^2 {J}}{\partial \pi_g \partial \pi_h}=0\ (g \neq h),
  % \end{equation}
  % thus
  \begin{equation}
    \nabla^2 {J}(\boldsymbol{\pi})=\left(\frac{1}{\eta}-\lambda \right) \cdot \operatorname{diag}\left(\frac{1}{\pi_1}, \ldots, \frac{1}{\pi_m}\right).
  \end{equation}
  When $0<\eta\lambda<1$, the Hessian matrix is positive definite in the interior, hence $J(\boldsymbol{\pi})$ is strictly convex with a unique minimizer. Since $\pi_g\to 0^+$ drives $J(\boldsymbol{\pi})\to +\infty$, the minimizer must lie strictly in the interior, \ie $\pi_g>0$. Next, we solve for the stationary point using the following Lagrangian function:
  \begin{equation}
  \mathcal{L}(\boldsymbol{\pi}, {\nu})=J(\boldsymbol{\pi})+{\nu}\left(\sum_g \pi_g-1\right)-\sum_g \mu_g \pi_g,
  \end{equation}
  with ${\nu}$ and $\mu_g$ are the multipliers of the equality and inequality constraints. Taking the partial derivative with respect to $\pi_g$ and setting it to zero yields the stationarity condition:
  \begin{equation}
    \frac{\partial \mathcal{L}}{\partial \pi_g}=-u_g^t-\lambda\left(1+\log \pi_g\right)+\frac{1}{\eta}\left(1+\log \pi_g-\log \pi_g^t\right)+\nu-\mu_g=0 .
  \end{equation}
  From the KKT complementary slackness condition, $\mu_g \pi_g=0$ and $\pi_g>0$ imply $\mu_g=0$. Define $\tau:=\frac{1}{\eta}-\lambda >0$. Simplifying the stationarity condition gives:
  \begin{equation}
    \tau \log \pi_g=u_g^t+\frac{1}{\eta} \log \pi_g^t+B,
  \end{equation}
  where $B$ is a constant. Exponentiating and applying the normalization condition $\sum_g \pi_g=1$ yields
  \begin{equation}
    \pi_g^{t+1}=\frac{\left(\pi_g^t\right)^{\frac{1}{1-\eta \lambda}} \exp \left(\frac{\eta}{1-\eta \lambda} u_g^t\right)}{\sum_j\left(\pi_j^t\right)^{\frac{1}{1-\eta \lambda}} \exp \left(\frac{\eta}{1-\eta \lambda} u_j^t\right)}, \quad 0<\eta\lambda<1,
  \end{equation}
  Equivalently,
  \begin{equation}
    \pi_g^{t+1} \propto\left(\pi_g^t\right)^{1 /(1-\eta \lambda)} \cdot \exp \left(u_g^t / \tau\right), \quad \tau=\frac{1-\eta \lambda}{\eta}.
  \end{equation}
  The above optimal solution can be interpreted as \textbf{a softmax distribution with a base distribution $\boldsymbol{\pi}^t$ and a temperature parameter $\tau=(1-\eta\lambda)/\eta$}. Therefore, as $\lambda$ increases, $\tau$ decreases, and the routing distribution becomes sharper.
\end{proof}

\subsection{Corollary 1}
\begin{corollary}[$\epsilon$-soft top-$k$ Approximation]
  \emph{There exists a threshold $\theta$ such that when $\lambda\geq \theta$, the entropy-driven routing mechanism approximates an $\epsilon$-soft top-$k$ routing.}
\end{corollary}
\begin{proof}
  Let $u_1\geq \cdots \geq u_m$ be arranged in descending order, and define $\delta_k:=u_k - u_{k+1} >0 \ (1\leq k\leq m)$. From Eq.~(\ref{eq: sovle}), the ratio satisfies
\begin{equation}
  \frac{\pi_{j}^{t+1}}{\pi_{k}^{t+1}} \leq\left(\frac{\pi_{j}^t}{\pi_{k}^t}\right)^{\frac{1}{1-\eta \lambda}} \exp \left(-\frac{\eta}{1-\eta \lambda} \delta_k\right) \quad (j>k).
\end{equation}
To achieve an approximate $epsilon$-soft top-$k$ routing scheme, the optimal routing weights should concentrate as much as possible on the top-$k$ experts, while the total weight assigned to the remaining $m-k$ experts should be minimized.
Since $\pi_{1}^{t+1}\geq\cdots\geq\pi_{m}^{t+1}$, we have
\begin{equation}
\sum_{j>k} \pi_{j}^{t+1} \leq(m-k) \pi_{k+1}^{t+1}, \quad \sum_{i \leq k} \pi_{i}^{t+1} \geq k \pi_{k}^{t+1}.
\end{equation}
Hence
\begin{equation}
  \begin{gathered}
    \sum_{j>k} \pi_j^{t+1} \leq(m-k) \pi_{k+1}^{t+1} \cdot \frac{\sum_{i \leq k} \pi_i^{t+1}}{k \pi_k^{t+1}} \leq \frac{m-k}{k} \frac{\pi_{k+1}^{t+1}}{\pi_k^{t+1}} \\
    \leq \frac{m-k}{k}\left(\frac{\pi_j^t}{\pi_k^t}\right)^{\frac{1}{1-\eta \lambda}} \exp \left(-\frac{\eta}{1-\eta \lambda} \delta_k\right) \\
    \leq \frac{m-k}{k} \exp \left(-\frac{\eta}{1-\eta \lambda} \delta_k\right)
    \end{gathered}
\end{equation}
Suppose there exists a small constant $\epsilon\in(0,1)$. As long as $\sum_{j>k} \pi_j^{t+1}\leq \epsilon$, the optimal routing distribution allocates at least a fraction $1-\epsilon$ of the probability mass to the top-$k$ experts, thereby making the entropy-driven routing mechanism approximate an $\epsilon$-soft top-$k$ routing. Hence, requiring
\begin{equation}
\sum_{j>k} \pi_{j}^{t+1} \leq \frac{m-k}{k}\exp \left(-\frac{\eta}{1-\eta \lambda} \delta_k\right) \leq\epsilon,
\end{equation}
yields the condition for approximation. For this inequality to provide a meaningful lower bound on the threshold, we need $\frac{k\epsilon}{m-k}<1$, yielding
\begin{equation}
\lambda\geq\frac{1}{\eta}+\frac{\delta_k}{\log\frac{k\epsilon}{m-k}}=\theta .
\end{equation}
\end{proof}

\section{Dataset Details}\label{app:datasets}
Additional dataset descriptions are provided below. And the statistical summaries of datasets are presented in Table~\ref{tab: datasets}.
\begin{itemize}[leftmargin=10pt]
    \item \textbf{Computers} and \textbf{Photo} are segments of the Amazon co-purchase graph, where nodes represent products, edges represent the co-purchased relations of products, and features are bag-of-words vectors extracted from product reviews.
    \item \textbf{Coauthor CS} and \textbf{Coauthor Physics} are co-authorship graphs based on the Microsoft Academic Graph from the KDD Cup 2016 challenge, where nodes represent authors, edge represent the corresponding authors have co-authored a paper, features consist of keywords from each author's published papers, and the class labels denote the most active research fields for each author.
    \item \textbf{Facebook} is a page-page graph of verified Facebook sites, where nodes correspond to official Facebook pages, links to mutual likes between sites, and features are extracted from the site descriptions. 
    \item \textbf{ogbn-arxiv} is a network dataset designed for predicting the subject areas of computer science arXiv papers. Each node represents a paper, and the directed edges indicate citation relationships between papers. The node features are 128-dimensional vectors obtained by averaging the word embeddings of the paper's title and abstract, where the embeddings are generated using the Skip-gram model over the MAG corpus. The task is to predict one of 40 subject areas (e.g., cs.AI, cs.LG) that are manually assigned by paper authors and arXiv moderators. The dataset is split by publication date, with training on papers published until 2017, validation on papers published in 2018, and testing on papers published since 2019.
    \item \textbf{Actor} is a network dataset designed for analyzing co-occurrence relationships among actors, where node represents an actor, and the edges between nodes indicate their co-occurrence on the same Wikipedia page. The node features are constructed from keywords extracted from the respective actors' Wikipedia pages.
    \item \textbf{Chameleon-fix} and \textbf{Squirrel-fix} are two page-page networks focusing on specific topics in Wikipedia, where nodes represent web pages, and edges denote mutual links between the pages. The node features are composed of informative nouns extracted from the corresponding Wikipedia pages. The task of these datasets is to categorize the nodes into five distinct groups based on the average monthly traffic received by each web page.
    \item \textbf{Tolokers} is a social network extracted from the Toloka crowdsourcing platform, where nodes represent workers and two workers are connected if they participate in the same task. The node features are constructed from the workers' profile information and task performance statistics, while the labels indicate whether a worker is banned in a project.
    \item \textbf{Roman-empire} is derived from the Roman Empire article on Wikipedia, where nodes in the dataset represent words from the article, edges indicating word dependencies. The node features are constructed from word embeddings obtained using the FastText method, and labels denote the syntactic roles of the words.
    \item \textbf{Penn94} is a Facebook social network, where nodes denote students and are labeled with the gender of users, edges represent the relationship of students. Node features
    are constructed from basic information about students which are major, second major/minor, dorm/house, year and high school.
\end{itemize}

\begin{table}[!htb]
\renewcommand\arraystretch{1.3}
\centering
\caption{Summary of datasets used}
\label{tab: datasets}
\resizebox{\linewidth}{!}{
    \begin{tabular}{c|cccc} 
\hline\hline
                 & Node Feature & Node Number & Edges   & Classes  \\ 
\hline
Computers        & 767          & 13,752       & 491,722  & 10       \\
Photo            & 745          & 7,650        & 238,162  & 8        \\
Coauthor CS      & 6,805         & 18,333       & 163,788  & 15       \\
Coauthor Physics & 8,415         & 34,493       & 495,924  & 5        \\
Facebook         & 128          & 22,470       & 342,004  & 4        \\
ogbn-arxiv       & 128          & 169,343      & 1,166,243 & 40       \\
Actor            & 932          & 7,600        & 30,019   & 5        \\
Chameleon-fix    & 2,325         & 890         & 13,584   & 5        \\
Squirrel-fix     & 2,089         & 2,223        & 65,718   & 5        \\
Tolokers         & 10           & 11,758       & 1,038,000 & 2        \\
Roman-empire     & 300          & 22,662       & 65,854   & 18       \\
Penn94           & 4,814         & 41,554       & 2,724,458 & 2        \\
\hline\hline
\end{tabular}
}
\end{table}

\section{Baseline Details}\label{app:baselines}
% \subsubsection{\textbf{Baseline Methods}}
To facilitate a better understanding of the baseline selection, some additional descriptions are provided as follows:
\begin{itemize}[leftmargin=10pt]
    \item \textbf{MLP} is a two-layer linear neural network that based on the original features of the nodes, without any propagation or aggregation rules.
    \item \textbf{GCN} is a neural network that aggregates information among neighboring nodes through message passing.
    \item \textbf{GAT} is a neural network that leverages multi-head attention to weight node features effectively on graph data.
    \item \textbf{SAGE} is a graph neural network that learns node representations by sampling and aggregating neighborhood information.
    \item \textbf{H2GCN} constructs a neural network by separating ego and neighbor embeddings, aggregating higher-order neighborhood information, and combing intermediate representations.
    \item \textbf{GPRGNN} is a graph neural network that optimizes node feature and topology extraction by adaptively learning Generalized PageRank weights.
    \item \textbf{FAGCN} is a novel graph convolutional network that integrates low and high-frequency signals through an adaptive gating mechanism.
    \item \textbf{ACMGCN} adaptively employs aggregation, diversification, and identity channels to extract richer local information for each node at every layer.
    \item \textbf{GloGNN} generates node embeddings by aggregating global node information and effectively captures homophily by learning a correlation matrix between nodes.
    \item \textbf{FSGNN} is a simplified graph neural network model that enhances node classification performance by introducing a soft selection mechanism.
    \item \textbf{LINKX} combines independent embeddings of the adjacency matrix and node features, generating predictions through a multi-layer perceptron and simple transformations.
    \item \textbf{ANS-GT} is a graph transformer architecture that effectively captures long-range dependencies and global context information through adaptive node sampling and hierarchical graph attention mechanisms.
    \item \textbf{NAGFormer} is a novel graph transformer that handles node classification tasks on large graphs by treating each node as a sequence aggregated from features of neighbors at various hops.
    \item \textbf{SGFormer} is a simplified and efficient graph transformer model that handles large-scale graph data through a single-layer global attention mechanism, achieving node representation learning with linear complexity.
    \item \textbf{Exphormer} is a novel sparse graph Transformer architecture designed to address the scalability issues faced by traditional graph Transformers when handling large-scale graph data. By introducing virtual global nodes and expander graphs, it achieves a sparse attention mechanism with linear complexity, demonstrating enhanced scalability on large-scale datasets.
    \item \textbf{Difformer} is a novel neural network architecture for learning complex dependencies between data instances. It uses an energy-constrained diffusion model to encode instances as dynamically evolving states, progressively integrating information. By optimizing a regularized energy function, the model derives the optimal diffusion strength between instances, enabling globally consistent representation learning.
    \item \textbf{GMoE} introduces a sparse MoE framework to GNNs, where each expert specializes in information aggregation over varying hop sizes, enabling nodes to dynamically select the most suitable receptive field.
    \item \textbf{NodeMoE} is a GNN framework that treats different graph spectral filters as experts and utilizes a Mixture-of-Experts approach to adaptively assign a specialized filter (e.g., low-pass or high-pass) to each node based on its local structural patterns.
    \item \textbf{DAMoE} is a Mixture-of-Experts framework designed to address the depth-sensitivity issue in graph-level tasks, where each expert is a full GNN model of a different depth, allowing the model to adaptively select the optimal aggregation depth for each individual graph.
\end{itemize}

% \subsubsection{\textbf{Criteria for Baseline Selection}}\label{app:standard-baseline}
% Firstly, we compare NCGCN with spatial methods that optimize message passing through rewiring, gating mechanisms, label propagation, and other techniques. These methods encompass classical approaches like GAT as well as SOTA methods tailor-made for addressing the heterophily problem, such as GBKSage and HOG. Additionally, we also compare our approach with the spectral method GPRGNN and local spectral methods FAGCN and ACM-GCN in solving the heterophily problem due to their widely validated effectiveness and code accessibility.

% As for ACM+ (using layer normalization) and ACM++ (using residual and adjacency matrix), we believe that these generic tricks fail to truly showcase the model's inherent capabilities, while incorporating adjacency matrices suffers from potential data leakage. For conciseness, we select ACM-GCN as our baseline approach. 
% The reason for choosing GloGNN instead of GloGNN++ is the same.

% Similar to NCGCN, UDGNN also proposes a framework based on separating nodes. However, the lack of open-source code and the complexity of this method hinder its reproducibility. Moreover, direct comparison with reported results is challenging due to different splitting strategies employed across datasets in the original paper.

\section{Paraneter Settings}\label{app:parameter}
\subsection{\textbf{Parameter Tuning}}
% Considering the scale of datasets, we use Tesla A100 40GB for model training and parameter tuning. To optimize hyperparameters according to validation accuracy for all methods, we use the open-source toolkit NNI~\cite{nni2021} along with its default TPE~\cite{bergstra2011algorithms} algorithm.

We used the Neural Network Intelligence (NNI) tool along with its default TPE algorithm for hyperparameter tuning to conduct experiments on the baseline models. The experiments were conducted using the same base parameters as our method, along with specific parameters unique to each baseline model. The special parameters are as follows:

% \subsubsection{\textbf{Criteria for Hyperparameter Tunning}}
% Based on our dataset selection criteria, many of the selected datasets have not been fully optimized by these baselines. To ensure a fair comparison, we provide detailed information on hyperparameter tuning as follows.

% The GBKSage and HOG methods exhibit significant computational inefficiency, requiring more than 100 times longer computation time compared to other approaches. Consequently, we only utilized their default parameters in our experiments. 
% As for CPGNN, reproducing its code under our dataset splits and random seeds demands substantial effort. Regrettably, we were unable to find an interface for hyperparameter tuning, such as adjusting the learning rate.

% For other methods, we fine-tune the common hyperparameters within the search space, including learning rate, weight decay, and dropout (as presented in Table~\ref{tab: common-para}). Since these methods often involve multiple specific hyperparameters and their reported optimal settings did not yield satisfactory results in our experimental setup, we adopt the default configurations provided in their respective source codes or papers. 
% Please refer to Appendix~\ref{app:hyper-baseline} for further details.
\begin{table}[!ht]
    \renewcommand\arraystretch{1.2}
    \centering
    \caption{Searche space for three common parameters of all methods.}
    \resizebox{0.8\linewidth}{!}{
    \begin{tabular}{ll} 
    \hline\hline
    \textbf{Parameter}        & \multicolumn{1}{c}{\textbf{Searche space}}                                            \\ 
    \hline
    learning rate    & \multicolumn{1}{c}{\{5e-3, 0.01, 0.05, 0.1\}}  \\
    weight decay     & \multicolumn{1}{c}{\{5e-5, 1e-4, 5e-4, 1e-3, 5e-3\}}                                             \\
    dropout          & \multicolumn{1}{c}{\{0.1, 0.3, 0.5, 0.7, 0.9\}}                             \\
    $\lambda$          & \multicolumn{1}{c}{\{0.001, 0.01, 0.1, 1.0\}}                             \\
    \hline\hline
    \end{tabular}}
    \label{tab: common-para}
\end{table}
\subsection{\textbf{Common Parameters}}\label{app:comm-para}
Table~\ref{tab: common-para} shows the searche space of three common parameters, including learning rate, weight decay and dropout rates. Table~\ref{tab: best_parameter_GCN}, Table~\ref{tab: best_parameter_SAGE}, and Table~\ref{tab: best_parameter_GAT} present the optimal parameter configurations of GNNMoE when using GCN-like P, SAGE-like P, and GAT-like P propagation operators.
% 表1，2，3展示了GNNMoE分别在使用GCN，SAGE，GAT-like P传播算子下的最优参数选择情况。

\subsection{{\textbf{Specific Parameters for Baselines}}}\label{app:hyper-baseline}
Here we provide an overview of the specific hyperparameter configurations for the baselines, all set in accordance with the specifications outlined in their respective papers.
\begin{itemize}[leftmargin=10pt]
  \item GloGNN: norm\_layers $\in \{1,2,3\}$, orders  $\in\{2,3,4\}$, $\alpha \in\{0.0,1.0\}$, $\beta \in \{0.1,1,10,100,1000\}$ ,$\gamma \in\{0,0.9\}$ with 0.1 interval and $\theta \in \{0,0.9,1.0\}$;
  \item FSGNN: aggregator $\in\{\text{cat}, \text{sum}\}$;
  \item FAGCN: $\epsilon \in \{0.2,0.3,0.4,0.5\}$
  \item GPRGNN: K $\in \{10\}$, dropout $\in \{0.5\}$, $\alpha \in{0.5}$;
  \item ACMGCN: variant $\in \{\text{False}\}$, is\_need\_struct $\in \{\text{False}\}$;
  \item H2GCN: num\_layers $\in \{1\}$, num\_mlp\_layers $\in \{1\}$;
  \item LINKX: is\_need\_struct $\in \{\text{False}\}$;
  \item ANS-GT: data\_augmentation $\in\{4,8,16,32\}$, n\_layer $\in\{2,3,4\}$ and batch size $\in\{8,16,32\}$;
  \item NAGFormer: hidden $\in\{128,256,512\}$, number of Transformer layers $\in\{1,2,3,4,5\}$ and number of propagation steps $\in\{7,10\}$;
  \item SGFormer: number of global attention layers is fixed as 1, number of GCN layers $\in\{1,2,3\}$, weight $\alpha$ $\in\{0.5,0.8\}$;
  \item Difformer: hidden $\in \{16,32,64\}$;
  \item GMoE: num\_layers $\in\{2,3\}$, loss coef $\in\{0.1,1\}$, num experts $\in\{4,8\}$, top $k\in\{1,2,4\}$, num experts 1hop$\in\{zero,half,all\}$;
  \item NodeMoE: balance $\in\{0, 0.001, 0.01, 0.1, 1\}$, gamma $\in\{0, 0.01, 0.1, 1\}$, num experts $\in\{2, 3, 5\}$, dropout $\in\{0, 0.5, 0.8\}$; 
  \item DAMoE: top $k\in\{2,3,4\}$, num experts$\in\{4,8\}$, loss coef $=0.001$, num\_layer$=3$, min layer=2;
\end{itemize}

% \newpage

\begin{table}[!htb]
\renewcommand\arraystretch{1.2}
\centering
\caption{Optimal Parameters for GNNMoE(P: GCN-like P)}
\label{tab: best_parameter_GCN}
\resizebox{\linewidth}{!}{
\begin{tabular}{c|cccc} 
\hline
\hline
             & learning rate    & weight decay       & dropout & lambda  \\ 
\hline
Computers    & 0.005 & 0.005    & 0.7     & 1       \\
Photo        & 0.05  & 5e-5 & 0.7     & 0.01    \\
CS           & 0.005 & 0.0001   & 0.3     & 0.001   \\
Physics      & 0.005 & 5e-5 & 0.7     & 0.1     \\
Facebook     & 0.005 & 1e-5 & 0.5     & 0.001   \\
ogbn-arixv   & 0.005 & 0.005    & 0.1     & 0.01    \\
Actor        & 0.005 & 1e-5 & 0.5     & 0.001   \\
chameleon    & 0.005 & 0.005    & 0.9     & 0.1     \\
squirrel     & 0.01  & 0.001    & 0.9     & 0.001   \\
tolokers     & 0.01  & 1e-5 & 0.1     & 0.001   \\
roman-empire & 0.01  & 1e-5 & 0.5     & 0.01    \\
Penn94       & 0.005 & 0.001    & 0.5     & 1       \\
\hline
\hline
\end{tabular}
}
\end{table}

\begin{table}[!htb]
\renewcommand\arraystretch{1.2}
\centering
\caption{Optimal Parameters for GNNMoE(P: SAGE-like P)}
\label{tab: best_parameter_SAGE}
\resizebox{\linewidth}{!}{
\begin{tabular}{c|cccc} 
\hline
\hline
             & learning rate & weight decay & dropout & lambda  \\ 
\hline
Computers    & 0.005         & 5e-5     & 0.7     & 0.01    \\
Photo        & 0.01          & 0.001        & 0.7     & 0.001   \\
CS           & 0.01          & 5e-5     & 0.5     & 0.001   \\
Physics      & 0.01          & 0.005        & 0.5     & 0.1     \\
Facebook     & 0.005         & 0.0001       & 0.5     & 0.001   \\
ogbn-arixv   & 0.005         & 0.0001       & 0.1     & 0.1     \\
Actor        & 0.005         & 5e-5     & 0.3     & 0.001   \\
chameleon    & 0.005         & 0.005        & 0.9     & 0.01    \\
squirrel     & 0.05          & 1e-5     & 0.7     & 0.1     \\
tolokers     & 0.005         & 0.005        & 0.7     & 0.01    \\
roman-empire & 0.005         & 0.0005       & 0.5     & 0.001   \\
Penn94       & 0.005         & 0.005        & 0.5     & 1       \\
\hline
\hline
\end{tabular}
}
\end{table}

\begin{table}[!htb]
\renewcommand\arraystretch{1.2}
\centering
\caption{Optimal Parameters for GNNMoE(P: GAT-like P)}
\label{tab: best_parameter_GAT}
\resizebox{\linewidth}{!}{
\begin{tabular}{c|cccc} 
\hline
\hline
             & learning rate & weight decay & dropout & lambda  \\ 
\hline
Computers    & 0.005         & 0.0001       & 0.5     & 0.01    \\
Photo        & 0.005         & 0.005        & 0.7     & 0.01    \\
CS           & 0.005         & 0.005        & 0.3     & 0.1     \\
Physics      & 0.005         & 5e-5     & 0.7     & 0.1     \\
Facebook     & 0.005         & 0.005        & 0.5     & 0.001   \\
ogbn-arixv   & 0.005         & 0.005        & 0.3     & 0.01    \\
Actor        & 0.005         & 0.005        & 0.5     & 0.01    \\
chameleon    & 0.01          & 0.0005       & 0.9     & 0.1     \\
squirrel     & 0.05          & 5e-5     & 0.5     & 0.1     \\
tolokers     & 0.01          & 0.0005       & 0.5     & 0.001   \\
roman-empire & 0.01          & 0.0001       & 0.5     & 0.001   \\
Penn94       & 0.01          & 0.0005       & 0.7     & 0.01    \\
\hline
\hline
\end{tabular}
}
\end{table}

\newpage
\section{More Experiment Results}
\subsection{Efficiency Comparison}\label{app:eff}
% 表VI和表VII展示了在Penn94和ogbn-arxiv上进行了效率对比实验，统计了早停的轮数和每一个epoch所消耗的时间。
Tables~\ref{tab: efficiency_penn94} and ~\ref{tab: efficiency_arxiv} present efficiency comparison experiments conducted on Penn94 and Ogbn-arxiv. The results report the number of epochs before early stopping and the time consumed per epoch.

\begin{table}[!htb]
\renewcommand\arraystretch{1.2}
\centering
\caption{Efficiency analysis summary on Penn94}
\label{tab: efficiency_penn94}
\resizebox{\linewidth}{!}{
\begin{tabular}{c|c|c} 
\hline\hline
\textbf{Time on Penn94} & Early Stop Epoch & Time of Every Epoch(s)  \\ 
\cline{1-1}\cline{2-3}
MLP            & 112              & 0.009~                  \\
GCN            & 446              & 0.230~                  \\
GAT            & 237              & 0.050~                  \\
GraphSAGE      & OOM              & OOM                     \\ 
\hline
H2GCN          & 265              & 0.518~                  \\
GPRGNN         & 669              & 0.018~                  \\
FAGCN          & 509              & 0.060~                  \\
ACMGCN         & 487              & 0.081~                  \\
GloGNN         & 281              & 0.114~                  \\
FSGNN          & 559              & 4.362~                  \\
LINKX          & 157              & 0.046~                  \\ 
\hline
Vanilla GT     & OOM              & OOM                     \\
ANS-GT         & OOM              & OOM                     \\
NAGphormer     & 331              & 0.140~                  \\
SGFormer       & 215              & 0.050~                  \\
Exphormer      & OOM              & OOM                     \\
Difformer      & OOM              & OOM                     \\ 
\hline
GCN-like P     & 123              & 0.190~                  \\
SAGE-like P    & 125              & 0.160~                  \\
GAT-like P     & 124              & 0.270~                  \\
\hline\hline
\end{tabular}}
\end{table}

\begin{table}[!htb]
\renewcommand\arraystretch{1.2}
\centering
\caption{Efficiency analysis summary on ogbn-arxiv}
\label{tab: efficiency_arxiv}
\resizebox{\linewidth}{!}{
\begin{tabular}{c|c|c} 
\hline\hline
\textbf{Time on ogbn-arixv} & Early Stop Epoch & Time of Every Epoch(s)  \\ 
\hline
MLP                & 335              & 0.014~                   \\
GCN                & 1174             & 0.051~                   \\
GAT                & 1036             & 0.072~                   \\
GraphSAGE          & 1026             & 0.054~                   \\ 
\hline
H2GCN              & OOM              & OOM                     \\
GPRGNN             & 293              & 0.101~                   \\
FAGCN              & 306              & 0.052~                   \\
ACMGCN             & 877              & 0.085~                   \\
GloGNN             & OOM              & OOM                     \\
FSGNN              & 494              & 1.525~                   \\
LINKX              & 115              & 0.035~                   \\ 
\hline
Vanilla GT         & OOM              & OOM                     \\
ANS-GT             & OOM              & OOM                     \\
NAGphormer         & 490              & 0.571~                   \\
SGFormer           & 492              & 0.082~                   \\
Exphormer          & 270              & 1.007~                   \\
Difformer          & OOM              & OOM                     \\ 
\hline
GCN-like P         & 575              & 0.181~                   \\
SAGE-like P        & 584              & 0.214~                   \\
GAT-like P         & 567              & 0.324~                   \\
\hline\hline
\end{tabular}
}
\end{table}

\subsection{More Results for RQ3}\label{app:weight}
Here we visualize the averaged expert-routing weight distributions assigned to all nodes on each dataset by the optimally tuned \model(SAGE-like P or GAT-like P) and its ablation model (w/o $\mathcal{L}_\text{route}$), as shown in Fig.~\ref{fig:weight-dis-sage} and Fig.~\ref{fig:weight-dis-gat}.

\newpage
\begin{figure*}[!htb]
  \centering
  \includegraphics[width =\linewidth]{FigAPP-weight-sage.pdf}
  \caption{Visualization of routing weight distributions before and after introducing the routing entropy regularization mechanism (P: SAGE-like P).}
  \label{fig:weight-dis-sage}
\end{figure*}

\begin{figure*}[!htb]
  \centering
  \includegraphics[width =\linewidth]{FigAPP-weight-gat.pdf}
  \caption{Visualization of routing weight distributions before and after introducing the routing entropy regularization mechanism (P: GAT-like P).}
  \label{fig:weight-dis-gat}
\end{figure*}

\newpage

\end{document}